%% file: paper.tex
\documentclass[]{bytedance_seed}

\usepackage[toc,page,header]{appendix}

\usepackage{minitoc}

\usepackage{tikz}
\usetikzlibrary{positioning, arrows.meta, decorations.pathreplacing, shapes.geometric}
\usepackage{amsmath}
\usepackage{amssymb}
\usepackage{mathtools}
\usepackage{amsthm}
\DeclareMathOperator*{\argmin}{arg\,min}
\newcommand{\x}[0]{{\mathbf{x}}}

\newcommand{\E}[0]{\mathbb{E}}

\newcommand{\Sk}[0]{\mathbf{S}}
\newcommand{\pop}{\mathcal{R}}
\theoremstyle{plain}
\newtheorem{theorem}{Theorem}[section]
\newtheorem{assumption}[theorem]{Assumption}

\title{Scaling Domain Data Repetition in LLM Pretraining}

\author[1, \ddagger]{Jingwei Li}
\author[1]{Xinran Gu}
\author[1]{Rui Dai}
\author[2, \dagger]{Xintong Hao}
\author[2]{Chengyin Xu}
\author[2]{Yan Wu}
\author[1]{Shuran Zheng}
\author[1, *]{Jingzhao Zhang}

\affiliation[1]{Tsinghua University}
\affiliation[2]{ByteDance Seed}

\contribution[\ddagger]{Work done at ByteDance Seed}
\contribution[\dagger]{Project Lead}
\contribution[*]{Corresponding Author}

\abstract{
As large language models scale, their training-token budgets must also increase to maintain an appropriate tokens-per-parameter ratio (\(\mathrm{TPP}\)). However, high-quality domain data is much harder to scale than general web data. As model size and the training-token budget increase, its fraction in the training mixture tends to decrease. Repeating the available high-quality data provides an effective way to counteract this dilution, but excessive repetition may lead to overfitting. We study this trade-off under practical LLM scaling, where the training-token budget grows proportionally with model size. 
For a fixed domain, we first find that, surprisingly at a fixed \(\mathrm{TPP}\), the optimal repetition count mildly increases with model size. 
Across different domains, we find that the optimal repetition count is strongly negatively correlated with the final validation loss of a domain: domains with lower loss can generally benefit from more repetitions. In contrast, the amount of unique domain data is only weakly related to the optimal repetition count. These findings suggest that repetition counts tuned on smaller proxy models with the same \(\mathrm{TPP}\) can provide a practical estimate for larger models.
}

\date{\today}
\correspondence{Jingzhao Zhang at \email{jingzhaoz@mail.tsinghua.edu.cn}}

\begin{document}
\maketitle


\input{sections/introduction}

\input{sections/relatedwork}

\input{sections/approach}
\input{sections/experiments}

\input{sections/ablation}
\input{sections/conclusion}

\clearpage

\bibliographystyle{plainnat}
\bibliography{main}

\clearpage

\beginappendix

\input{appendix/details}

\input{appendix/experiments_4}

\input{appendix/experiments_5}

\input{appendix/proof}

\end{document}

%% file: sections/introduction.tex
\section{Introduction}

Large language models (LLMs) are often pretrained on mixtures of diverse domains, including general web text, code, mathematics, scientific corpora, and multilingual content~\citep{doddapaneni2025primer, taylor2022galactica, team2023gemini, dubey2024llama}. As model size increases, compute-optimal training also requires an increasing number of training tokens~\citep{kaplan2020scaling, hoffmann2022training}. This creates an asymmetric data-scaling problem: broad general web data~\cite{raffel2020exploring, penedo2024fineweb, li2024datacomp} can often be scaled relatively easily, whereas high-quality domain-specific data~\cite{paster2024openwebmath, li2023starcoder, guo2020wiki} is much harder to scale at the same rate. 
However, if the amount of such data remains fixed, its mixing ratio decreases as the total training budget grows. Recent work shows that under data mixing, knowledge-dense domains may only be learned once their mixing ratio exceeds a critical threshold; below this threshold, additional training can still fail to acquire the target knowledge~\cite{gu2025datamixinginducephase}. Therefore, maintaining a sufficient fraction of high-quality domain data is important when scaling data.

Prior work has extensively studied data repetition, where examples from a finite dataset are reused multiple times during training~\cite{muennighoff2023scaling, lovelace2026prescriptive, sedova2026scaling, wu2026data}. For high-quality domain data, repetition offers a direct way to counteract data dilution, but may also increase the risk of overfitting~\cite{lee2022deduplicating, carlini2022quantifying, hernandez2022scaling, xue2023repeat}. We therefore need to study the trade-off between the domain-learning gains from repeated high-quality data and the risk of overfitting, and to identify the optimal number of repetitions for different domains.
  
A first line of work studies data repetition in the single-dataset
setting~\citep{muennighoff2023scaling, xue2023repeat,
lovelace2026prescriptive}. More recent work extends this question to mixtures of general and domain-specific data, showing that general data can regularize repeated domain data and developing scaling laws that account for repetition under data mixing~\citep{sedova2026scaling, liu2026infolaw, chudnovsky2026internal}. Across these settings, the conclusion is that larger models are more susceptible to overfitting from repeated data, suggesting that repetition should be minimized when training large LLMs. We confirm this  with our experiments shown on the left in Figure~\ref{fig:motivation}.

\begin{figure*}[t]
\centering
\begin{tabular}{cc}
\includegraphics[width=0.48\textwidth]{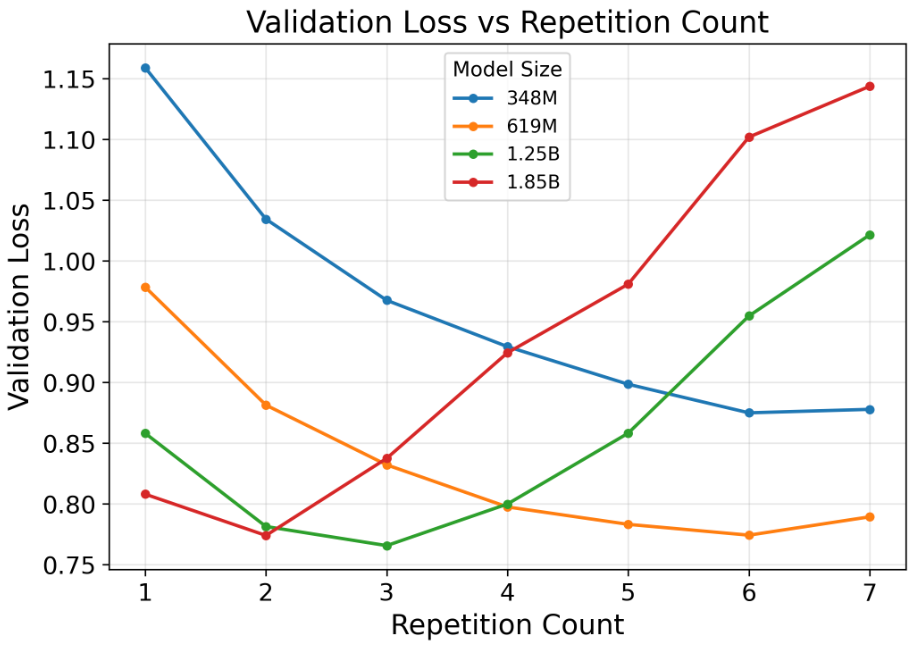} &
\includegraphics[width=0.48\textwidth]{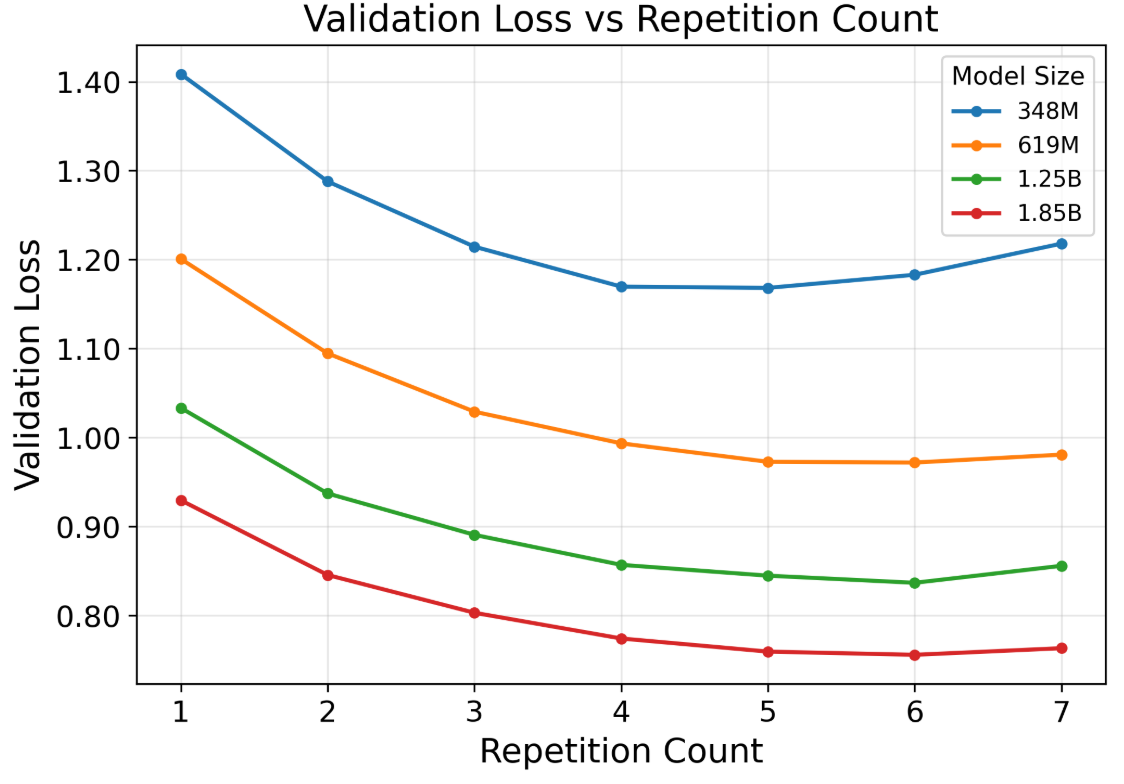} \\
\end{tabular}
  \caption{\textbf{Data repetition under two model-scaling setups.}
  \textbf{Left:} When the training-data size is fixed independently of model
  size, the optimal repetition count decreases with model size.
  \textbf{Right:} When the tokens-per-parameter ratio \(\mathrm{TPP}\) is fixed,
  the optimal repetition count increases with model size. The fixed-\(\mathrm{TPP}\) setting better reflects practical
  LLM scaling and is the focus of our study.}
  \label{fig:motivation}
\vspace{-0.3cm}
\end{figure*}

However, we also note that existing cross-scale comparisons typically hold the training-data size fixed rather than preserving the tokens-per-parameter ratio~\cite{muennighoff2023scaling,xue2023repeat,chudnovsky2026internal},
\(\mathrm{TPP} \coloneqq D/N\), where \(D\) is the total number of training
tokens and \(N\) is the model size. 
This distinction leads to different conclusions about repetition. As illustrated in Figure~\ref{fig:motivation}, when \(D\) is fixed, increasing the model size exposes a larger model to the same repeated dataset, causing overfitting to occur earlier and the optimal repetition count to approach one. 
\textit{In contrast, when the training-token budget grows with model size using a fixed
\(\mathrm{TPP}\)}~\citep{hoffmann2022training, kalra2026quantifying}, \textit{the optimal repetition count surprisingly increases with model size.}

To study repetition under this scaling regime, we consider four high-quality
domains: code, math, Wikipedia and medical data. For
each domain, we vary its amount of unique data and repetition count while
keeping the total training-token budget fixed for each model size. This setup
allows us to isolate how data repetition interacts with domain properties,
unique-data size, and model scale.

Our main findings are as follows.

\begin{itemize}
\item \textbf{The optimal repetition count is strongly negatively
correlated with the final validation loss of a domain.}
Domains with lower validation loss can generally tolerate and benefit from
more repetitions, whereas domains with higher validation loss overfit
earlier.

\item \textbf{At a fixed tokens-per-parameter ratio \(\mathrm{TPP}\), the optimal
repetition count increases with model size.}
This trend is opposite to that observed when the data budget is fixed
across model sizes. It also supports conservative transfer from smaller
proxy models: a repetition count that does not cause overfitting in a
smaller model is unlikely to cause overfitting in a larger model trained at
the same \(\mathrm{TPP}\).

\item \textbf{The optimal repetition count is largely insensitive to the amount of unique data.}
Across the tested fractions of unique high-quality tokens, the repetition count that minimizes validation loss remains nearly unchanged. Therefore, when estimating the optimal repetition count on a smaller proxy model, it is not necessary to use a specific unique high-quality token fraction; any representative fraction within the tested range can be used.

\end{itemize}

Together, these results provide a practical approach for configuring
the optimal repetition counts for high-quality domains. Repetition counts can first be swept on a
smaller proxy model with the same \(\mathrm{TPP}\) for an arbitrary fraction of unique high-quality tokens, and the result can be used to select relatively safe repetition counts for the target model.

%% file: sections/relatedwork.tex
\section{Related Work}

\textbf{Data Repetition in Single-Dataset Training.} Early studies mainly examine multi-epoch training on a single dataset. One line of work extends the Chinchilla scaling law~\cite{hoffmann2022training} and develops scaling laws for repeated data. \cite{muennighoff2023scaling} model the validation loss under repetition using exponential decay and show that a few epochs of repetition can perform similarly to fresh data, while heavier reuse yields diminishing returns. \cite{lovelace2026prescriptive} further extend this formulation by modeling the subsequent increase in validation loss with an additive overfitting penalty, achieving more accurate predictions. Another line of work studies the factors that determine repetition-induced degradation. \cite{xue2023repeat} identify dataset size, model size, and training objective as key factors, while \cite{charton2024emergent} show that repetition can improve generalization on some synthetic mathematical tasks. \cite{yan2025larger} theoretically prove that larger datasets can generally support more reuse. Beyond exact duplication, \cite{kazdan2026scale} find that semantic duplication can also cause overfitting, and \cite{wu2026data} show that sparse models benefit from more repeated epochs than dense models. However, these studies mainly repeat an entire dataset, whereas practical LLM pretraining often repeats only selected domains within data mixtures.

\textbf{Data Repetition in Pretraining Mixtures.}
More recent work studies repetition within pretraining mixtures, where a limited subset is reused while the remaining data remain unique. \cite{hernandez2022scaling} study a setting in which 10\% of the training tokens are drawn from a repeatedly reused subset and the remaining 90\% are unique. They find that repetition can cause non-monotonic degradation, with intermediate repetition levels being particularly harmful. Extending this analysis, \cite{chudnovsky2026internal} jointly vary the subset size and repetition count, showing that a moderately sized subset repeated a moderate number of times can be more harmful than either a larger subset repeated fewer times or a smaller subset repeated more times. \cite{sedova2026scaling} study mixture pretraining with limited target-domain data and abundant generic data. They show that generic data mitigates overfitting from repeated target-domain data and that the optimal repetition count depends on the amount of target data, compute budget, and model scale. However, existing cross-scale studies do not preserve the tokens-per-parameter ratio. We instead study repetition under a fixed \(\mathrm{TPP}=D/N\), such that the training-token budget grows proportionally with model size.

\textbf{Data Mixture Optimization for LLMs.}
Data mixture optimization studies how to allocate a fixed training budget across domains by modeling their relative value and interactions. Existing methods estimate domain weights using proxy training, predictive models, or scaling laws, and transfer the resulting mixture to larger models~\citep{xie2023doremi, fan2024doge, liu2024regmix, diao2025climb, ye2025data, kang2024autoscale, shukor2025scalinglawsoptimaldata}. These methods generally assume sufficient unique data in each domain and therefore do not model repetition. Recent work~\citep{liu2026infolaw, sedova2026scaling} relaxes this assumption by deriving repetition-aware mixture scaling laws that jointly characterize domain allocation and data reuse.

%% file: sections/approach.tex
\section{Setup}

\begin{figure}[t]
\vspace{-0.1cm}
\centering
\includegraphics[width=0.97\textwidth]{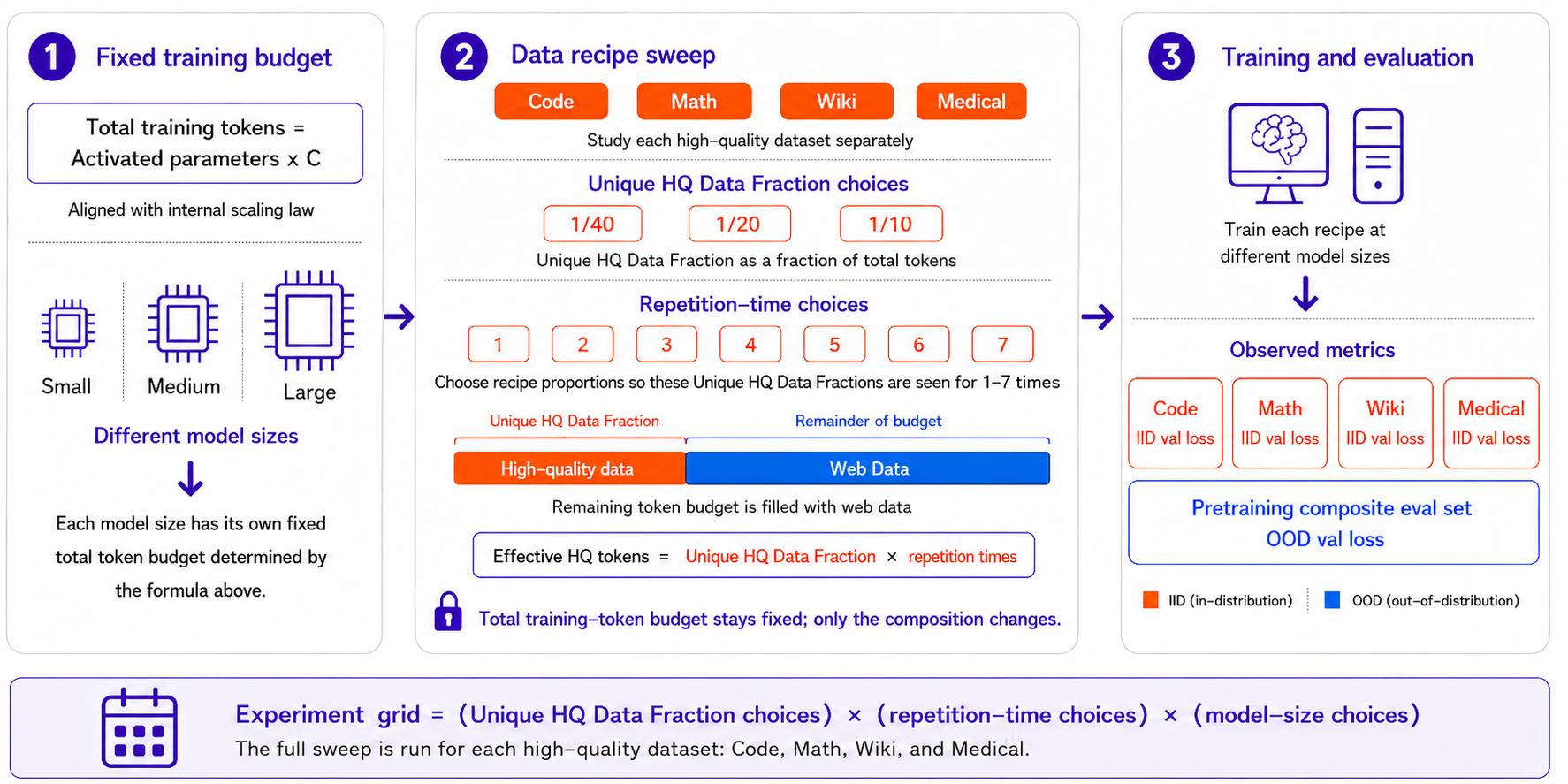}
\caption{\textbf{Overview of our experimental framework.} We train models of
different sizes at a fixed tokens-per-parameter ratio and vary the amount and
repetition count of high-quality domain data in the training mixture. We
evaluate both in-domain validation loss and out-of-domain performance.}
\label{fig:framework}
\vspace{-0.5cm}
\end{figure}
We study data repetition under different model sizes. Let
\(\mathcal{N}=\{N_1,\ldots,N_K\}\) denote the model sizes. For a model of size
\(N\in\mathcal{N}\), we set its total training-token budget to
\begin{align*}
    D_N = \mathrm{TPP} \cdot N,
\end{align*}
where \(\mathrm{TPP}\) denotes the number of training tokens per parameter and is set to a constant greater than 100 in our experiments. Therefore, all training
runs with the same model size use exactly the same number of training tokens,
while larger models receive proportionally larger token budgets.
We focus on different high-quality domains. Let
$$
\mathcal{D}_{\mathrm{HQ}}
=
\{
\textsc{Code},
\textsc{Math},
\textsc{Wiki},
\textsc{Medical}
\}
$$
denote the set of high-quality domains. For each high-quality domain \(d\in\mathcal{D}_{\mathrm{HQ}}\), we sweep over different combinations of unique-data fractions and repetition counts. 

Next, we set the number of unique high-quality tokens $U_{N,\alpha}$ as
\begin{align*}
        U_{N,\alpha} = \alpha D_N
, \quad \alpha \in
    \left\{
    \frac{1}{40},
    \frac{1}{20},
    \frac{1}{10}
    \right\},
\end{align*}
and the repetition count as $e \in \{1,2,3,4,5,6,7\}.$
In other words, for a configuration \((d,N,\alpha,e)\), we first select a fixed subset from domain \(d\) containing $U_{N,\alpha}$ unique tokens.
This subset is then repeated \(e\) times. 
Hence, the total number of high-quality tokens during training is
\begin{equation*}
    H_{N,\alpha,e}
    =
    eU_{N,\alpha}
    =
    e\alpha D_N.
\end{equation*}
The remaining token budget is filled with non-repeated web
data, whose token count is
\begin{equation*}
    W_{N,\alpha,e}
    =
    D_N-H_{N,\alpha,e}
    =
    \left(1-e\alpha\right)D_N.
\end{equation*}
All configurations satisfy \(e\alpha\leq 1\), and thus
\(W_{N,\alpha,e}\geq 0\). Importantly, \(\alpha\) controls the amount of
unique high-quality data, whereas \(e\) controls how often this fixed subset is
revisited. Their product \(e\alpha\) determines the final proportion of
high-quality token presentations in the training stream. The total budget
\(D_N\) remains unchanged across recipes for a fixed \(N\); only its
composition between the target high-quality domain and the web corpus changes.
We do not mix multiple high-quality domains in the same run.

Let \(\theta_{d,N,\alpha,e}\) denote the model obtained from configuration
\((d,N,\alpha,e)\). The complete experiment grid for each domain $\mathcal{G}_d$ and the full set of experiments $\mathcal{G}$ are
\begin{equation*}
    \mathcal{G}_d
    =
    \mathcal{N}\times\mathcal{A}\times\mathcal{E}, \quad
    \mathcal{G}
    =
    \bigcup_{d\in\mathcal{D}_{\mathrm{HQ}}}
    \left(\{d\}\times\mathcal{G}_d\right).
\end{equation*}

We evaluate each trained model using both in-distribution and
out-of-distribution validation losses. For a validation corpus
\(\mathcal{V}\), we define the token-averaged negative log-likelihood as
\begin{equation*}
    \mathcal{L}(\theta;\mathcal{V})
    =
    -\frac{1}{M_{\mathcal{V}}}
    \sum_{\boldsymbol{x}\in\mathcal{V}}
    \sum_{t=1}^{|\boldsymbol{x}|}
    \log p_{\theta}
    \left(x_t\mid x_{<t}\right),
\end{equation*}
where \(\boldsymbol{x}=(x_1,\ldots,x_{|\boldsymbol{x}|})\) is a token sequence
and $    M_{\mathcal{V}}
    =
    \sum_{\boldsymbol{x}\in\mathcal{V}}
    |\boldsymbol{x}|$
is the total number of validation tokens. For a run targeting domain \(d\), the
in-distribution metric is
\begin{equation*}
    \mathcal{L}_{\mathrm{IID}}^{(d)}(N,\alpha,e)
    =
    \mathcal{L}
    \left(
    \theta_{d,N,\alpha,e};
    \mathcal{V}_{d}^{\mathrm{IID}}
    \right),
\end{equation*}
where \(\mathcal{V}_{d}^{\mathrm{IID}}\) is the held-out validation set from
the same high-quality domain. We additionally report
\begin{equation*}
    \mathcal{L}_{\mathrm{OOD}}^{(d)}(N,\alpha,e)
    =
    \mathcal{L}
    \left(
    \theta_{d,N,\alpha,e};
    \mathcal{V}^{\mathrm{OOD}}
    \right),
\end{equation*}
where \(\mathcal{V}^{\mathrm{OOD}}\) is a pretraining validation set
used to measure performance outside the repeated target domain. Our setup is illustrated in Figure~\ref{fig:framework}.

%% file: sections/experiments.tex
\section{Analysis of the Optimal Repetition Count}
\label{sec:results}

In this section, we summarize the empirical factors associated with the optimal repetition count and provide a theoretical explanation. Empirically, the optimal repetition count depends mildly on model size, is largely insensitive to the fraction of unique high-quality data, and is strongly negatively correlated with the minimum validation loss. We then use a theoretical model to explain these observations through the trade-off between knowledge acquisition and noise fitting.

\subsection{Optimal Repetition Counts for High-Quality Domains}
\label{sec:optimal_repetition}

Recall that for each model size \(N\), we train with a token budget
\(D_N=\mathrm{TPP}\cdot N\) and record the validation loss at the end of
training. We vary the fraction of unique high-quality data, denoted by
\(\alpha\), and its repetition count, denoted by \(e\), while keeping the
total number of training tokens fixed. We use these experiments to study
how the optimal repetition count depends on $\alpha$, model size, and data domain.

\begin{figure*}[t]\centering
\begin{tabular}{ccc}
\includegraphics[width=0.313\textwidth]{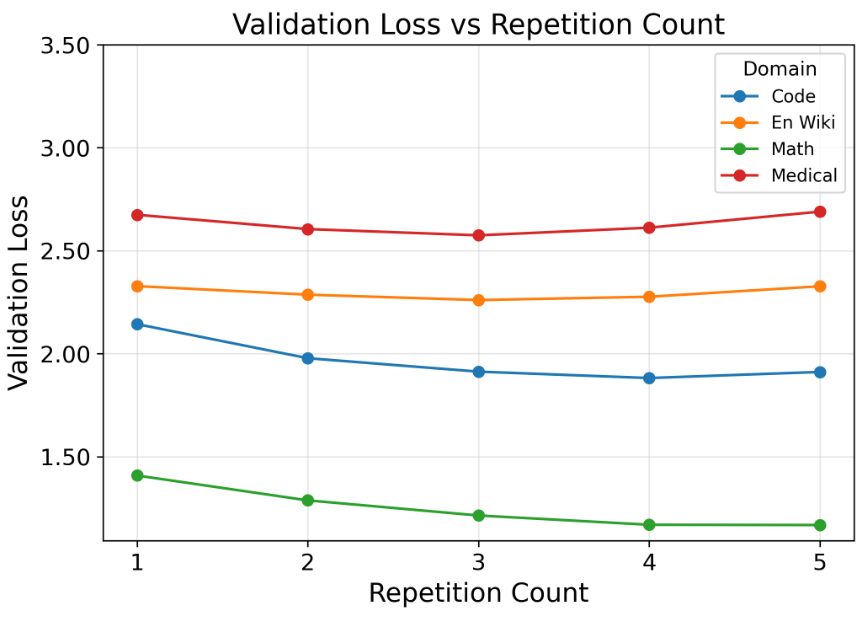} &
\includegraphics[width=0.313\textwidth]{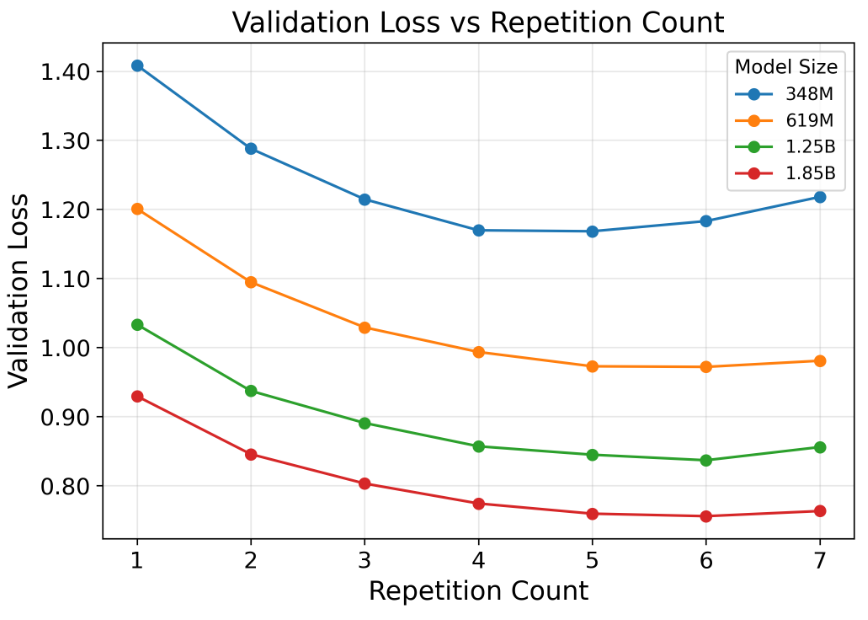} &
\includegraphics[width=0.313\textwidth]{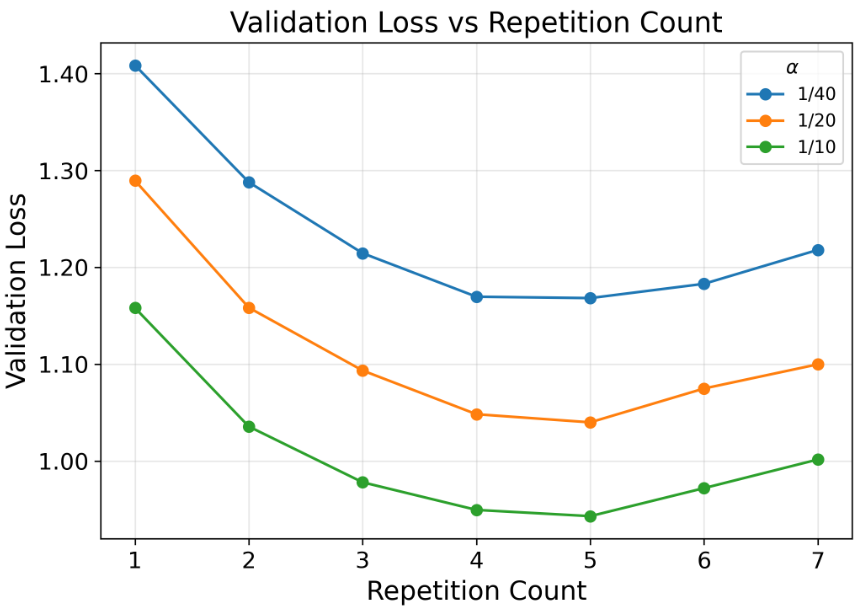} \\
\end{tabular}
\caption{
\textbf{Factors affecting the optimal repetition count.}
Left: different high-quality domains exhibit substantially different optimal repetition counts.
Middle: for the Math domain, the optimal repetition count increases with model size under a fixed TPP.
Right: for the Math domain, varying $\alpha$ mainly shifts the validation loss but has little effect on the optimal repetition count.
Overall, the domain has the strongest effect on the optimal repetition count, followed by model size, while $\alpha$ has little effect.
}
\label{fig:three_aspects}
\vspace{-0.3cm}
\end{figure*}

Figure~\ref{fig:three_aspects} compares the effects of data domain, model size, and the fraction of unique high-quality data $\alpha$ on the optimal repetition count. The left panel shows a strong domain dependence. Math reaches its minimum at around 5 repetitions, while Wiki, Code, and Medical reach their minima earlier, indicating that the optimal repetition count varies substantially across domains. The middle panel shows the effect of model size on the Math domain at a fixed $\alpha$. As model size increases, the optimal repetition count shifts toward larger values, suggesting a moderate dependence on model scale. The right panel shows the effect of $\alpha$ on the Math domain at a fixed model size. Although increasing $\alpha$ consistently reduces the validation loss, the location of the minimum remains nearly unchanged, indicating that $\alpha$ has little effect on the optimal repetition count. Overall, the optimal repetition count mainly depends on the data domain and model size, but is largely insensitive to $\alpha$. The full results across all domains, model sizes, and values of $\alpha$ are provided in Appendix~\ref{app:sec4}.

To quantify the above observations, we estimate the optimal repetition count for each combination of domain \(d\), model size \(N\), and the fraction of unique high-quality data \(\alpha\), and then quantitatively characterize its dependence on these factors. Specifically, we fit a quadratic function to the final validation loss as a function of repetition count:
\begin{equation*}
    \widehat{L}_{d,N,\alpha}(e)
    =
    a_{d,N,\alpha}e^2
    +
    b_{d,N,\alpha}e
    +
    c_{d,N,\alpha}.
\end{equation*}
We use the minimum of the fitted curve as the estimated optimal repetition count,
\begin{equation*}
    \widehat{e}^{*}_{d,N,\alpha}
    =
    -\frac{b_{d,N,\alpha}}{2a_{d,N,\alpha}}.
\end{equation*}
This continuous estimate reduces the effect of the discrete repetition grid, allowing clearer analysis of the correlations between the optimal repetition count and different factors.

Since different domains exhibit substantially different validation-loss levels, we further characterize the domain effect through validation loss. Specifically, instead of treating the domain only as a categorical variable, we use the minimum validation loss achieved within each domain setting as a continuous proxy for its domain-specific characteristics, and quantitatively examine how it relates to the optimal repetition count.

\begin{figure}[t]
\vspace{-0.1cm}
\centering
\includegraphics[width=0.97\textwidth]{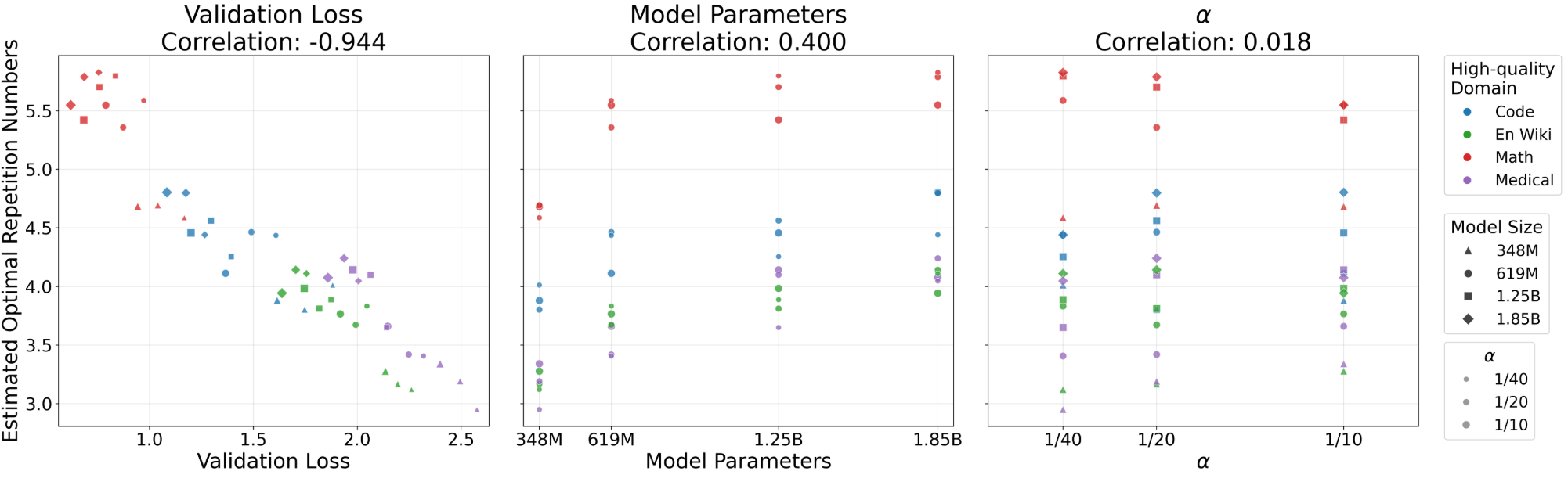}
\caption{\textbf{Factors associated with the estimated optimal repetition count.} For each combination of high-quality domain, model size, and unique data fraction, we fit a quadratic function to the final validation loss as a function of repetition count and use the minimum of the fitted curve as the estimated optimum. From left to right, we plot the estimated optimal repetition count against the minimum validation loss, model size, and unique high-quality data fraction. The corresponding Pearson correlations are \(-0.944\), \(0.400\), and \(0.018\), respectively. Colors denote high-quality domains, marker shapes denote model sizes, and marker sizes denote unique data fractions where applicable.}
\label{fig:optimal_repetition_correlation}
\vspace{-0.1cm}
\end{figure}

Figure~\ref{fig:optimal_repetition_correlation} compares the estimated optimal repetition count with the minimum validation loss, model size, and unique high-quality data fraction. The optimal repetition count has a strong negative correlation with the minimum validation loss, with a Pearson correlation of \(-0.944\). Across the domains considered in our experiments, settings with higher validation loss tend to prefer fewer repetitions, whereas settings with lower validation loss tend to have larger optimal repetition counts. Model size has a weaker positive correlation of \(0.400\), consistent with the observation that larger models tend to prefer slightly more repetitions under a fixed tokens-per-parameter ratio. By comparison, the correlation with the unique high-quality data fraction is only \(0.018\), indicating that the optimal repetition count is nearly insensitive to the amount of unique data over the range considered.

Overall, the optimal repetition count is mainly determined by the high-quality domain loss, with a smaller dependence on model size and little observable dependence on the unique high-quality data fraction. 

These findings suggest a practical procedure for selecting the repetition count for large-scale training: we first train a smaller proxy model with the same tokens-per-parameter ratio as the target model, using any suitable value of $\alpha$. For each high-quality domain, the validation loss provides a useful indication of the range of repetition counts worth considering. This value can then serve as a conservative estimate for the target model, since our experiments show that \textit{a repetition count that does not cause overfitting on the proxy model also remains safe for a larger model under the same tokens-per-parameter ratio.}

\subsection{Theoretical Analysis}
\label{sec:theorem}
In this section, we analyze the experimental observations above from a theoretical perspective. To capture the experimental setup, we consider a one-hot linear regression problem in $\mathbb{R}^{\infty}$.
Let $\{e_k\}_{k\ge1}$ be the standard basis, where each coordinate $k$ represents a knowledge unit, and its frequency follows a power law. Specifically, following~\cite{yan2025larger}, let the input distribution be $\mathbb{P}(\x=e_k)=p_k,$ where $p_k=c_\alpha k^{-\alpha},    c_\alpha=\zeta(\alpha)^{-1},     \text{ and }   \alpha>1.$ Further, let $\theta^*$ be the target linear predictor, where $\theta^*$ specifies the response associated with knowledge unit $k$. We make the following assumption:
\begin{assumption}[Source Condition]
\label{ass:source-condition}
We assume that $\theta^*$ satisfies a prior that $\E[(\theta_k^*)^2 p_k]=k^{-\beta}, \beta \in (1,+\infty).$
\end{assumption}
Source condition assumptions are widely used in theoretical work such as~\cite{li2025functional,lin2024scaling}.
The total weighted signal energy is finite because $\sum_{k=1}^{\infty}\E[p_k(\theta_k^*)^2]     =\sum_{k=1}^{\infty}k^{-\beta}<\infty.$

Next, we define the learning setup. Given a training-token budget $D$ and model size $N$, let $
    \Sk_N:\mathbb{R}^{\infty}\to\mathbb{R}^{N}
$
be the projection onto the first $N$ coordinates. Due to the limited model capacity, we can only observe $\Sk_N \x$ instead of observing the intact $\x$~\cite{dai2026explaining,li2025functional}. Therefore, we observe the dataset $
    \mathcal D=\{(\Sk_N\x_i,y_i)\}_{i=1}^{D},
$
with the following data distribution and loss,
\begin{align*}
        \x_i\overset{\mathrm{iid}}{\sim}p,
    \qquad
    y_i=\langle \x_i,\theta^*\rangle+\varepsilon_i,
    \qquad
    \varepsilon_i\overset{\mathrm{iid}}{\sim}\mathcal{N}(0,\sigma^2),
    \qquad \sigma^2>0, \\
        L_N(\theta)
    =
    \frac{1}{2D}\sum_{i=1}^{D}
    \left(\langle \Sk_N\x_i,\theta\rangle-y_i\right)^2,
    \qquad
    \theta\in\mathbb{R}^{N}.
\end{align*}
To fit the linear regression model, we initialize at $\theta_0=\mathbf{0}$ and run full-batch gradient descent:
$$\theta_{r+1}=\theta_r-\eta\nabla L_N(\theta_r),     \qquad     r=0,1,2,\ldots,     \qquad     0<\eta<1.$$
To evaluate the performance, we define population risk and the total expected risk. For $\theta\in\mathbb{R}^{N}$, define the population excess risk
$$\pop_N(\theta)     :=     \frac12\E_{\x}     \left[         \left(             \langle \Sk_N\x,\theta\rangle             -\langle \x,\theta^*\rangle         \right)^2     \right].$$
To investigate the expected population risk with respect to $r$, we write
$\pop_{D,N}(r;\beta,\sigma)     :=     \E_{\theta^*,\mathcal D,\varepsilon}[\pop_N(\theta_r)],$
and define the earliest optimal integer stopping time by
$$r^*(D,N;\beta,\sigma)     :=     \min\argmin_{r\in\mathbb Z_{\ge0}}     \pop_{D,N}(r;\beta,\sigma).$$
When $\beta$ and $\sigma$ are fixed, we simply write $r^*(D,N)$.

Condition on the number of observations of each knowledge unit gives the exact decomposition:
$$\begin{aligned}
2\pop_{D,N}(r;\beta,\sigma)={}&\underbrace{\sum_{k>N}k^{-\beta}}_{\text{unrepresented knowledge}}+\underbrace{\sum_{k=1}^{N}k^{-\beta} \E_{m\sim B(D,p_k)}(1-\eta m/D)^{2r}}_{\text{knowledge-acquisition error}}\\
&+\underbrace{\sigma^2\sum_{k=1}^{N}p_k \E_{m\sim B(D,p_k)} \mathbf{1}_{\{m>0\}}\cdot \frac{(1-(1-\eta m/D)^r)^2}{m}}_{\text{noise-fitting error}},
\end{aligned}$$ where $B(D,p_k)$ denotes the Binomial distribution. Repeated optimization decreases the knowledge-acquisition error but increases the noise-fitting error, and the optimal repetition count is the point at which the latter marginal effect begins to dominate.

\begin{theorem}[Noise decay]
\label{thm:noise-decay}
For fixed token size $D$ and model size $N$, let $0<\sigma_1^2\le\sigma_2^2$. Then
\[
    r^*(D,N;\beta,\sigma_2)
    \le
    r^*(D,N;\beta,\sigma_1).
\]
Moreover, as $\sigma^2\to 0$,
\[
    r^*(D,N;\beta,\sigma)
    =
    \frac{\log(1/\sigma^2)}{-\log(1-\eta/D)}+O(1).
\]
\end{theorem}
\paragraph{Findings:}A smaller $\sigma^2$ reduces the noise-fitting term without changing the knowledge-acquisition term, so the benefit of fitting the signal dominates for more iterations; overfitting still occurs, but its onset is delayed. Theorem~\ref{thm:noise-decay} formalizes this mechanism and explains the negative relationship between validation loss and the optimal repetition count in Figure~\ref{fig:optimal_repetition_correlation}. 

\begin{theorem}[Stopping time and model size]
\label{thm:model-size-stopping}
Suppose $\beta > \alpha$. For fixed $D$, $\beta$, and $\sigma^2>0$, define
\[
    N_0(D,\sigma^2)
    =
    \left(
        \frac{D(2-\eta/D)}{c_\alpha\eta\sigma^2}
    \right)^{\frac{1}{\beta-\alpha}}.
\]
Then
\[
    r^*(D,N+1)\le r^*(D,N),
    \qquad
    N>N_0(D,\sigma^2).
\]
The crossover scale satisfies
\[
    N_0(D,\sigma^2)
    =
    \Theta\!\left(
        \left(\frac{D}{\sigma^2}\right)^{1/(\beta-\alpha)}
    \right).
\]

\end{theorem}
\paragraph{Findings:} Under a fixed token budget, increasing $N$ extends the model toward rarer and weaker knowledge units without providing additional observations, causing the noise-fitting effect to dominate earlier once the model passes the explicit crossover scale in Theorem~\ref{thm:model-size-stopping}, as illustrated in the left panel of Figure~\ref{fig:motivation}.

\begin{theorem}[Stopping time under linear data--model scaling]
\label{thm:linear-scaling}
Suppose $\beta>\alpha$. For fixed $\sigma^2>0$ and $C_0>0$, suppose that
\[
    \frac{D}{N}\longrightarrow C_0.
\]
Then, as $D,N\to\infty$,
\[
    r^*(D,N)
    =
    \Theta\!\left(
        D^{\alpha/\beta}
    \right).
\]
\end{theorem}
\paragraph{Findings:} Under fixed tokens-per-parameter scaling, however, the data budget grows together with model size, allowing the reduction in knowledge-acquisition error to dominate for longer; Theorem~\ref{thm:linear-scaling} quantifies the resulting increase in the optimal repetition count shown in the right panel of Figure~\ref{fig:motivation}.

All proofs are deferred to Appendix~\ref{app:proof}.

%% file: sections/ablation.tex
\begin{figure*}[tp]
\centering
\includegraphics[width=0.972\textwidth]{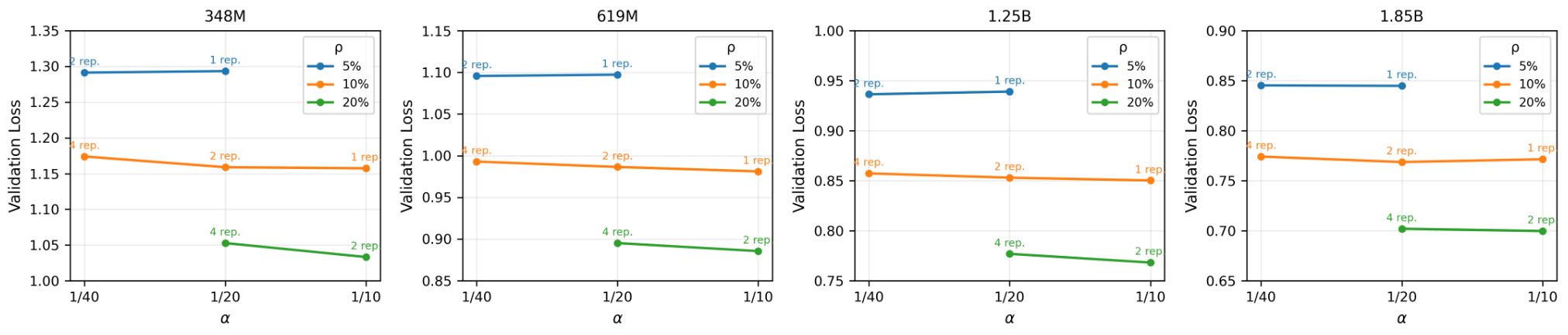} \\[4pt]
\includegraphics[width=0.972\textwidth]{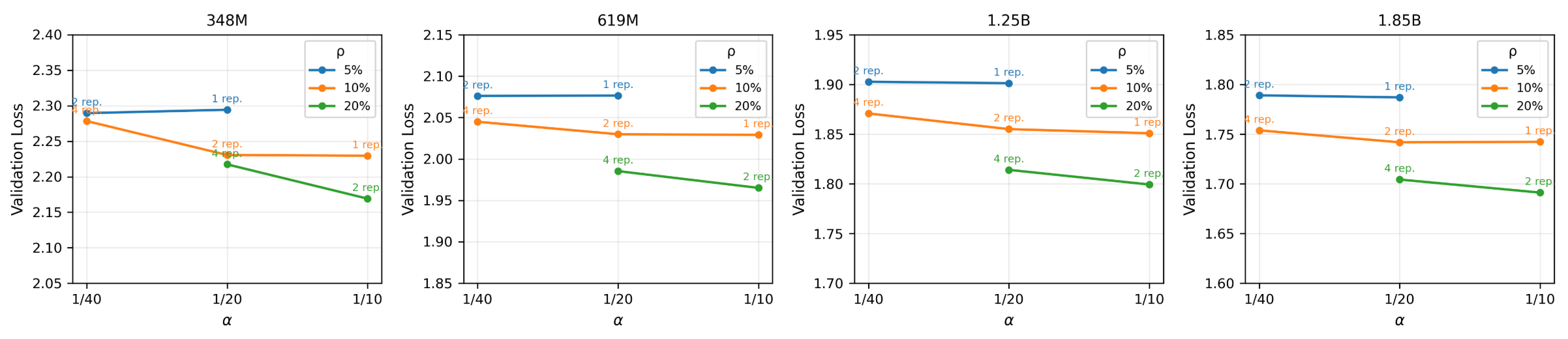}
\caption{
\textbf{Final validation loss at fixed total high-quality data fractions.}
Top: Math. Bottom: Wiki. We fix the total high-quality-data fraction and vary the repetition count, such that more repetition corresponds to fewer unique high-quality tokens. More unique data consistently yields lower validation loss in both domains. However, Math is relatively robust to repetition, whereas Wiki degrades sharply when the repetition count increases beyond 2.
}
\label{fig:fixed_concentration_main}
\vspace{-0.3cm}
\end{figure*}

\section{Additional Analyses}
\label{sec:ablation}
In this section, we present additional analyses of data repetition. Section 5.1 compares repeated and unique high-quality tokens at a fixed total domain fraction and shows that the effect of repetition is strongly domain dependent. Section 5.2 examines OOD pretraining performance and finds that replacing unique high-quality tokens with repeated tokens has only a limited effect when the total high-quality and web-data fractions are fixed. Finally, Section 5.3 examines whether the optimal repetition count depends on the learning-rate schedule.

\subsection{Unique versus Repeated Tokens at a Fixed Domain Fraction}
\label{sec:fixed_concentration}

In the previous section, we fix the amount of unique high-quality data and vary its repetition count. We now consider the complementary setup studied by \cite{muennighoff2023scaling}, where the total number of training tokens from a dataset is fixed and repeated tokens are directly compared with the same number of unique data. Their results suggest that, for up to roughly four repetitions, repeated data can achieve performance close to that of an equivalent amount of unique data. However, this analysis does not distinguish how this trade-off varies across data domains.

We therefore study the same question separately for different high-quality domains. Let $\alpha$ denote the fraction of unique high-quality data and $e$ its repetition count. The total fraction of high-quality tokens consumed during training is
$$
\rho = \alpha e.
$$
For a fixed $\rho$, increasing $e$ requires decreasing the amount of unique data according to
$
\alpha = \rho / e.
$
This setup allows us to directly ask whether repeating a smaller unique dataset $e$ times can match the performance of using $e$ times as much unique data once, under the same total high-quality token budget.

Figure~\ref{fig:fixed_concentration_main} shows the results for Math and Wiki. We find that the result depends strongly on the domain. For Math, increasing the repetition count from 1 to 4 causes only a small increase in validation loss. Thus, repeated Math data can largely substitute for additional unique Math data, consistent with the observation of \cite{muennighoff2023scaling}. In contrast, Wiki exhibits a clear degradation as repetition increases. While repeating the data twice incurs only a modest loss increase, heavier repetition leads to substantially worse validation loss, indicating that repeated Wiki tokens cannot effectively replace the corresponding amount of unique data.

These results show that the effectiveness of repeated data under a fixed token budget is domain dependent, rather than being characterized by a universal repetition threshold. Additional results in Appendix~\ref{app:sec5} show the same contrast across the remaining domains. Both Code and Medical show increasing validation loss as more unique data is replaced by repeated data, although the magnitude of degradation varies across domains.

\begin{figure*}[tp]
\centering
\includegraphics[width=0.972\textwidth]{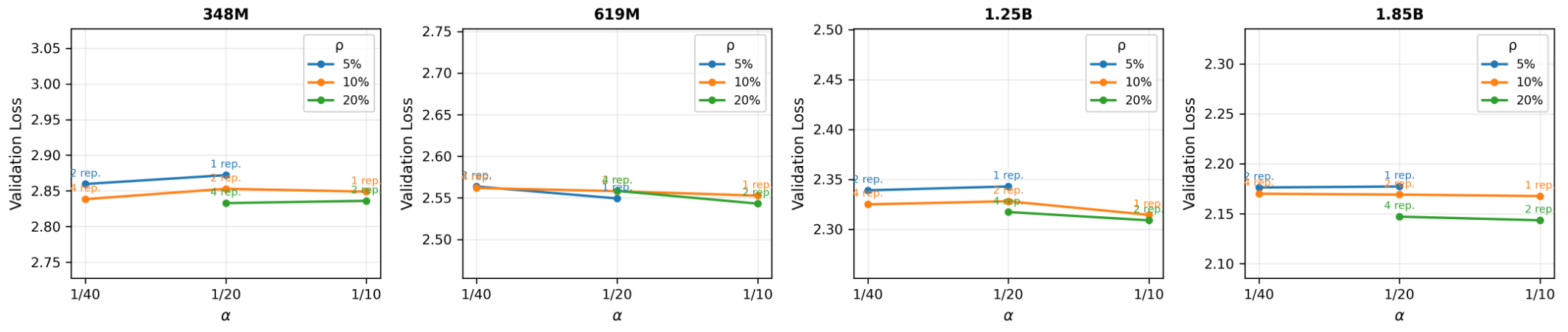} \\[4pt]
\includegraphics[width=0.972\textwidth]{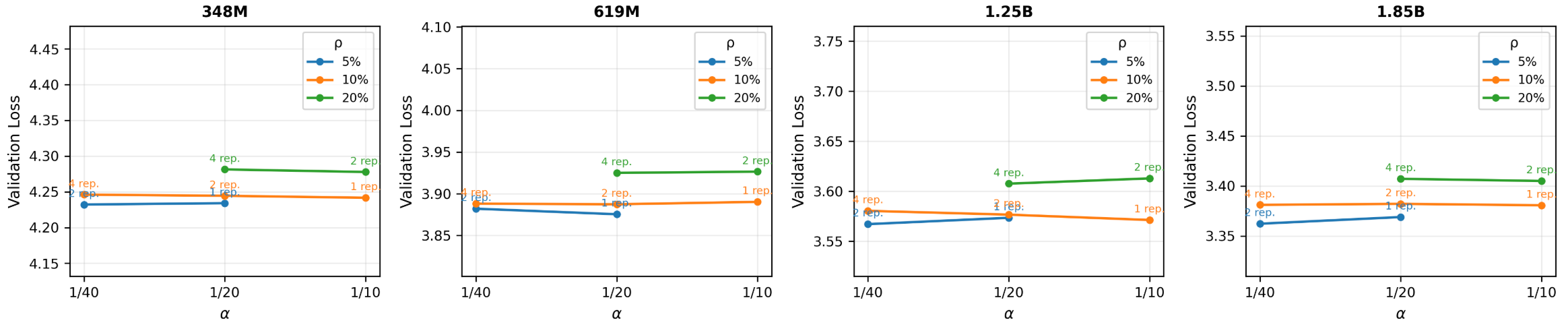}
\caption{
\textbf{OOD validation loss at fixed total high-quality data fractions.}
Top: ArXiv. Bottom: News. Each line fixes the total fraction of high-quality tokens, while the annotations indicate the corresponding repetition counts. Unlike the in-domain validation loss, the OOD loss changes only slightly as unique tokens are replaced with repeated tokens and shows no consistent monotonic relationship with the repetition count.
}
\label{fig:fixed_concentration_ood}
\vspace{-0.3cm}
\end{figure*}

\subsection{Effect of Repetition on OOD Pretraining Performance}
\label{sec:ood_results}

Next, we study whether repetition within a high-quality domain affects pretraining performance outside that domain. To isolate the effect of repetition, it is important to control the amount of web data used during training. If we instead fix the amount of unique high-quality data and increase its repetition count, the total amount of high-quality data increases accordingly, leaving fewer training tokens for web data. Any change in OOD performance would then confound the effect of repetition with the change in web data.

We therefore adopt the fixed-domain-fraction setup from Section~\ref{sec:fixed_concentration}. Specifically, we fix the total fraction of high-quality Math tokens, and hence also the amount of web data, while varying the repetition count by changing the amount of unique high-quality data. This allows us to isolate how replacing unique high-quality tokens with repeated tokens affects performance outside the repeated domain.

We evaluate the trained models on two OOD pretraining validation sets, ArXiv and News. Figure~\ref{fig:fixed_concentration_ood} reports the results. In contrast to the clear domain-dependent degradation observed on in-domain validation sets, the OOD validation loss remains largely stable as the repetition count changes. Across model sizes and total high-quality data fractions, the differences between configurations using 1, 2, or 4 repetitions are generally small.

These results suggest that, when the total fractions of high-quality and web tokens are fixed, replacing unique high-quality tokens with repeated tokens mainly affects performance within the repeated domain, while having limited effect on OOD pretraining performance.

\begin{figure*}[tp]
    \centering
    \setlength{\tabcolsep}{2pt}
    \begin{tabular}{ccc}
        \includegraphics[width=0.315\textwidth]{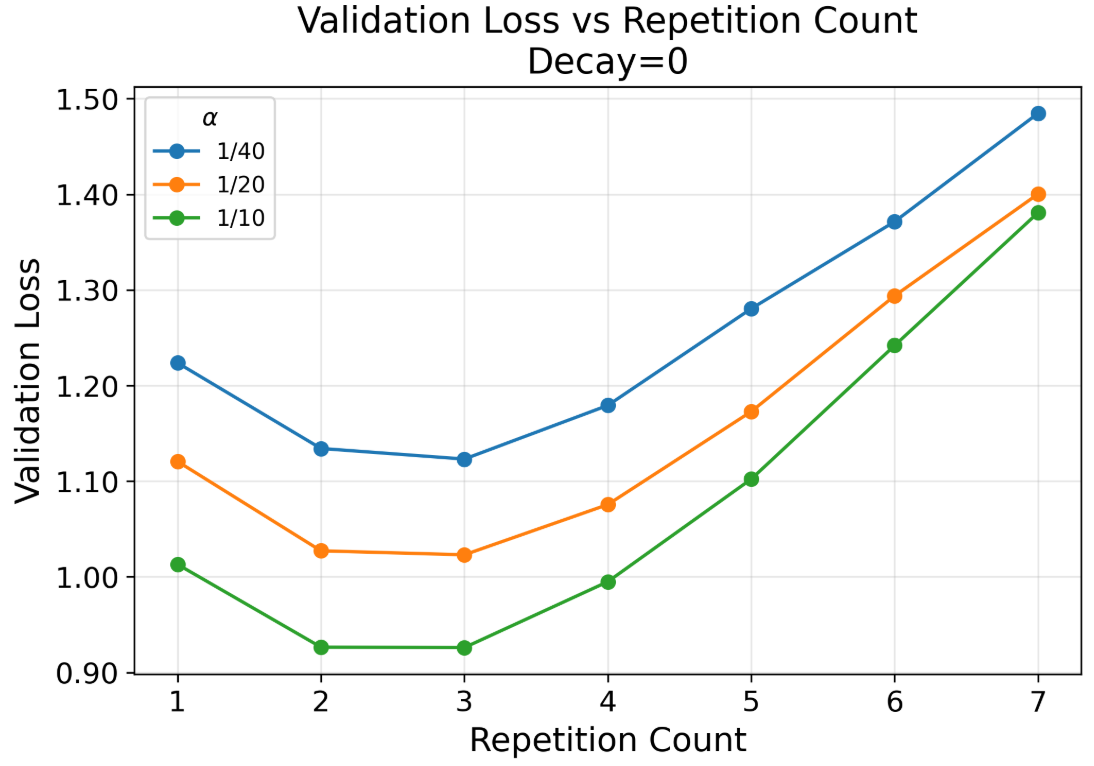} &
        \includegraphics[width=0.315\textwidth]{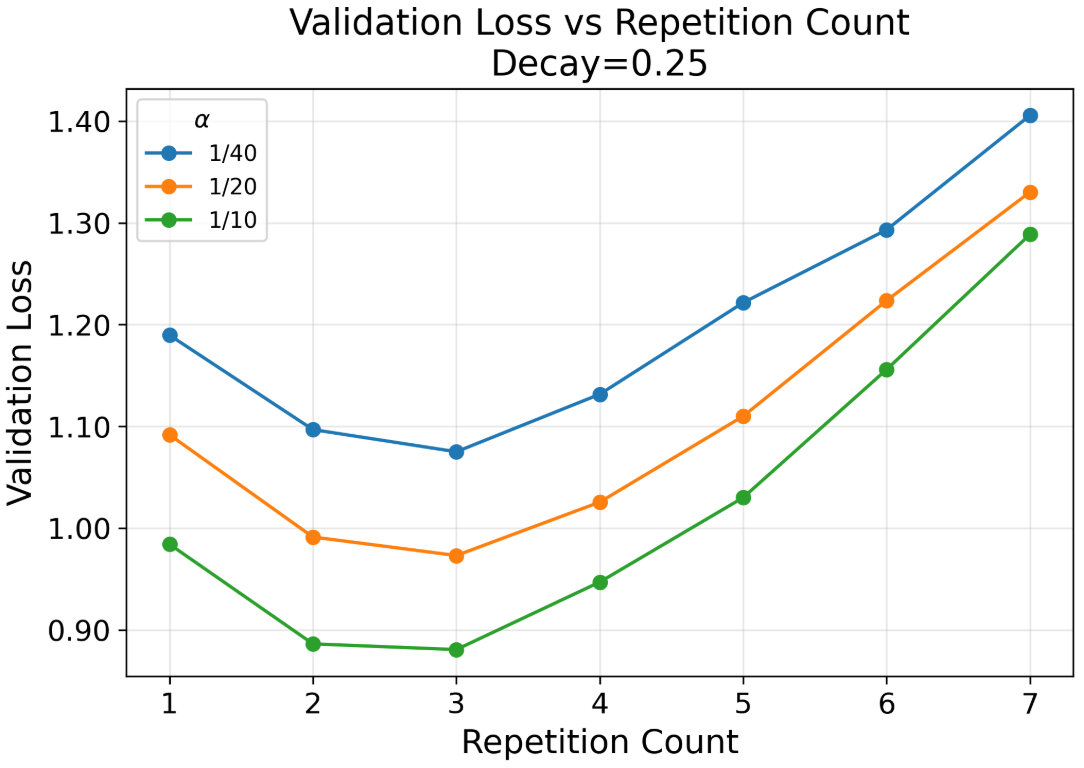} &
        \includegraphics[width=0.315\textwidth]{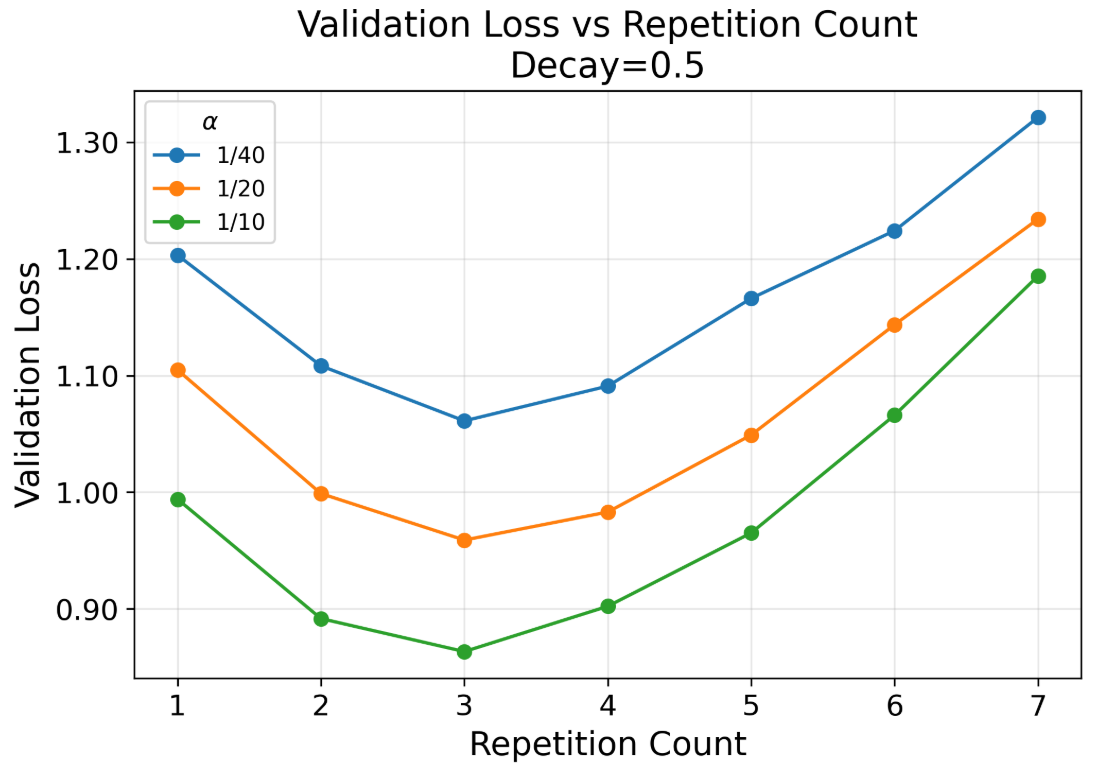} \\[-2pt]
        \multicolumn{3}{c}{
            \includegraphics[width=0.315\textwidth]{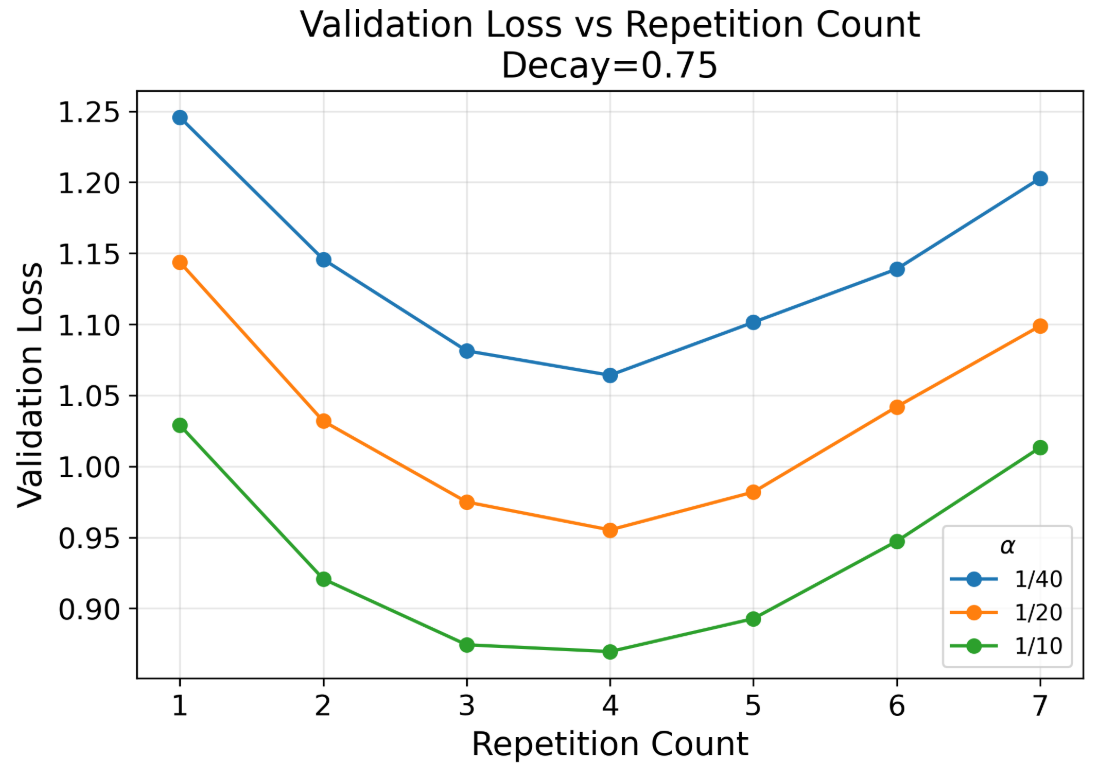}
            \hspace{2pt}
            \includegraphics[width=0.315\textwidth]{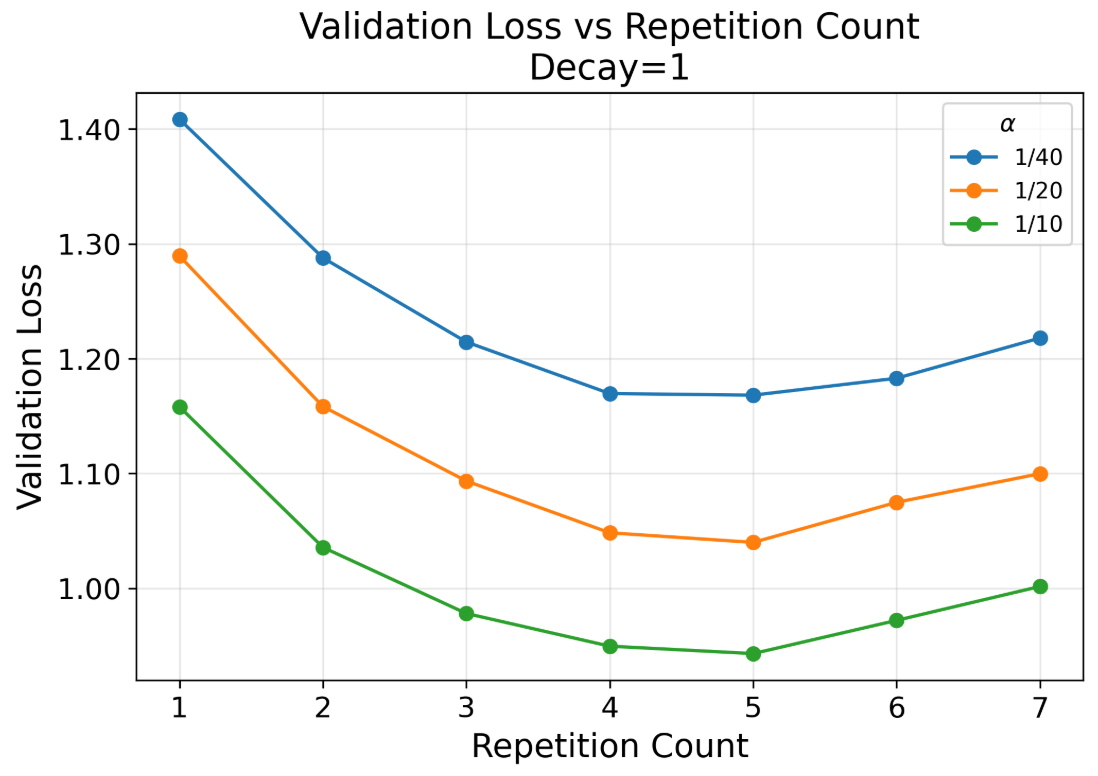}
        }
    \end{tabular}
    \caption{\textbf{Effect of the learning-rate schedule on the optimal repetition count.} From left to right: WSD with decay starting at $0\%$, $25\%$, $50\%$, and $75\%$ of training, followed by a constant learning rate. Earlier decay leads to degradation after fewer repetitions, while delaying or removing the decay allows more repetitions. }
    \label{fig:lr_schedule}
    \vspace{-0.2cm}
\end{figure*}

\subsection{Effect of Various Learning-Rate Schedules}
\label{sec:lr_schedule}

Finally, we examine whether the optimal repetition count depends on the learning-rate schedule. We compare a warmup-stable-decay (WSD) schedule with the decay phase starting at \(0\%\), \(25\%\), \(50\%\), or \(75\%\) of training, as well as a constant learning rate. For each schedule, we vary the fraction of unique high-quality data and the repetition count while keeping all other training settings unchanged.

As shown in Figure~\ref{fig:lr_schedule}, the learning-rate schedule has a clear effect on the optimal repetition count. Earlier decay causes the validation loss to increase after fewer repetitions, whereas delaying the decay shifts this increase to a larger repetition count. With a constant learning rate, the model tolerates the largest number of repetitions before degradation occurs.

A possible explanation is that repeated samples seen during the low-learning-rate stage can be fitted more closely, increasing the tendency to memorize sample-specific patterns. Delaying the decay keeps training at a relatively high learning rate for longer, which may reduce this effect through stronger optimization noise and implicit regularization. Overall, although excessive repetition eventually degrades performance across all schedules, the learning-rate schedule determines how much repetition the model can tolerate before this degradation begins.

%% file: sections/conclusion.tex
\section{Conclusion and Future Directions}

In this work, we study domain data repetition across model scales while keeping the tokens-per-parameter ratio fixed. We find that the optimal repetition count is strongly negatively correlated with the minimum validation loss of the repeated domain: domains with lower validation loss generally support more repetitions. At a fixed tokens-per-parameter ratio, the optimal repetition count increases mildly with model size, while remaining nearly insensitive to the fraction of unique high-quality tokens over the range considered. These observations suggest that repetition counts selected on smaller proxy models with the same TPP can provide conservative estimates for larger models. Our theoretical analysis further provides an explanation for the observed dependence on validation loss and model scale.

Our current setup repeats only one high-quality dataset in each training run. Future work should consider mixtures in which multiple domains are repeated simultaneously and study their interactions. Another direction is to develop a quantitative scaling rule that predicts the optimal repetition count of a large model from small-model results, enabling more accurate and efficient data-recipe selection.

%% file: appendix/details.tex
\section{General Implementation Details}
In this section, we describe the datasets and training hyperparameters used in our experiments.

\subsection{Details of Datasets}
Our pretraining corpus consists of general web data and several high-quality domain-specific datasets. The general web data contains content from diverse sources, covering a broad range of topics and writing styles. The high-quality data includes four domains:

\begin{itemize}

\item \textbf{Code data.}
A multilingual programming corpus covering a range of programming languages and software-development scenarios.

\item \textbf{Math data.}
A bilingual mathematical corpus containing educational materials, mathematical problems, and step-by-step solutions across different difficulty levels.

\item \textbf{Wiki data.}
Encyclopedic text covering diverse concepts, entities, and factual knowledge.

\item \textbf{Medical data.}
A collection of medical and health-related content covering biomedical knowledge, clinical topics, and health education.

\end{itemize}

\subsection{Details of Training Hyperparameters}
\label{app:training_hyperparameters}

The total training-token budget was scaled proportionally with model size under a fixed tokens-per-parameter ratio. Following \cite{bi2024deepseek}, we set the learning rate and global batch size for different model sizes according to power-law functions of the model size. The learning rate was linearly warmed up for the first 200 optimization steps and was then held constant for the remainder of training. We used an attention dropout rate of \(0.1\). Optimization used the Muon optimizer~\citep{liu2025muon}, a variant of AdamW.

%% file: appendix/experiments_4.tex
\section{Additional Results on Section~\ref{sec:optimal_repetition}}
\label{app:sec4}
In this section, we present additional results of Section~\ref{sec:optimal_repetition}.

Figures~\ref{fig:math_repetition},~\ref{fig:wiki_repetition},~\ref{fig:code_repetition}, and~\ref{fig:medical_repetition} present the complete results for Math, Wiki, Code, and Medical. Among the four domains, Math supports the most repetition, with an optimal repetition count of about 5--6, followed by Code and Wiki, while Medical has the lowest optimal repetition count of about 3--4. For each domain, the optimal repetition count increases with model size. In contrast, for a fixed model size, the optimal repetition count remains nearly unchanged across different values of $\alpha$. Overall, all experimental results are consistent with the observations in the main text.

\begin{figure*}[h]\centering
\begin{tabular}{cccc}
\includegraphics[width=0.23\textwidth]{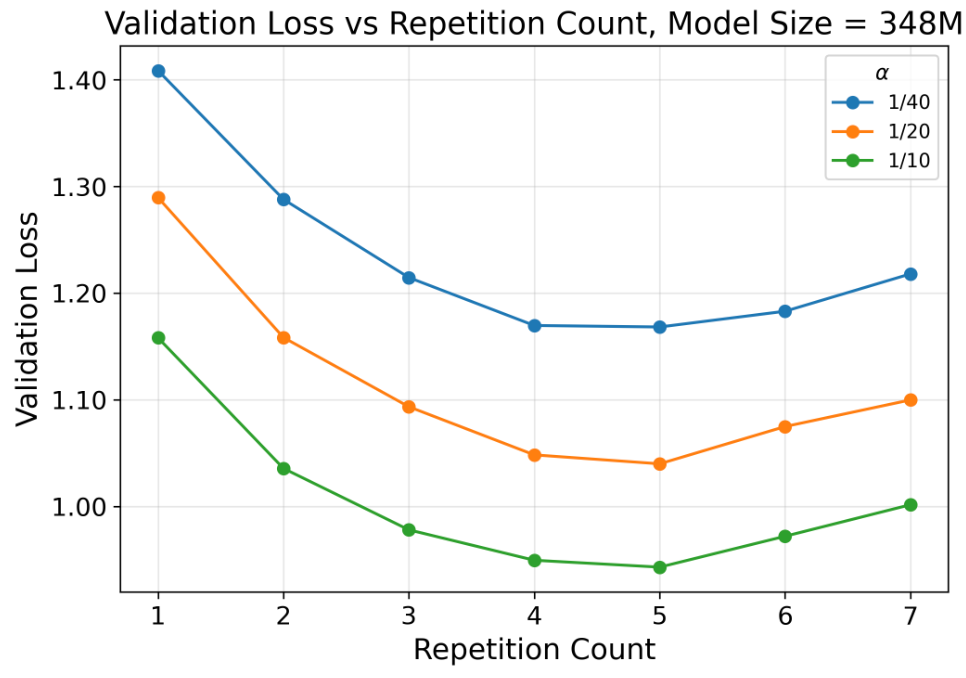} &
\includegraphics[width=0.23\textwidth]{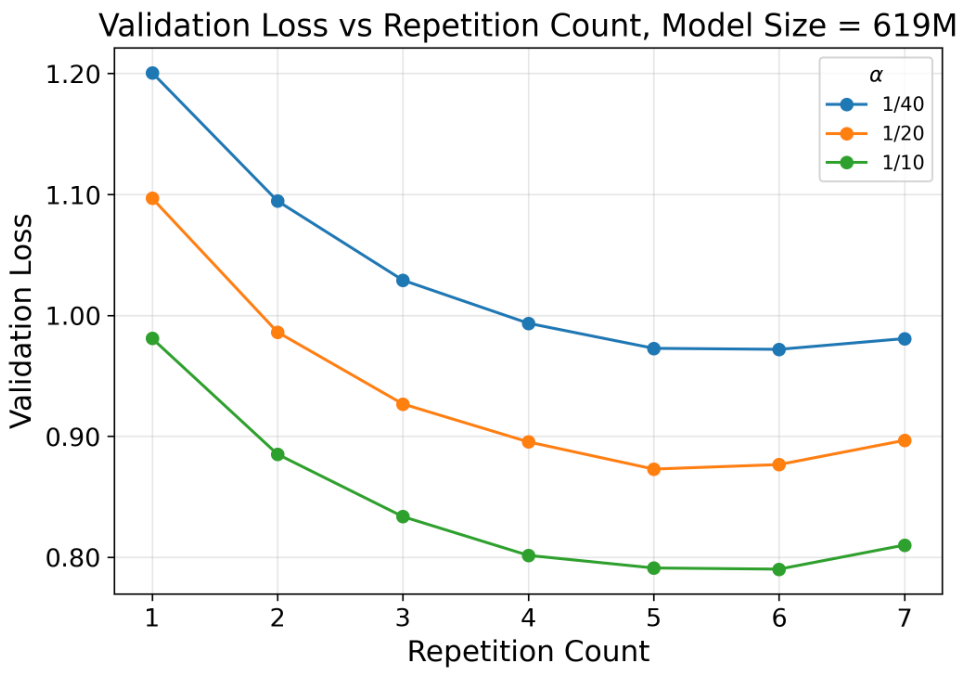} &
\includegraphics[width=0.23\textwidth]{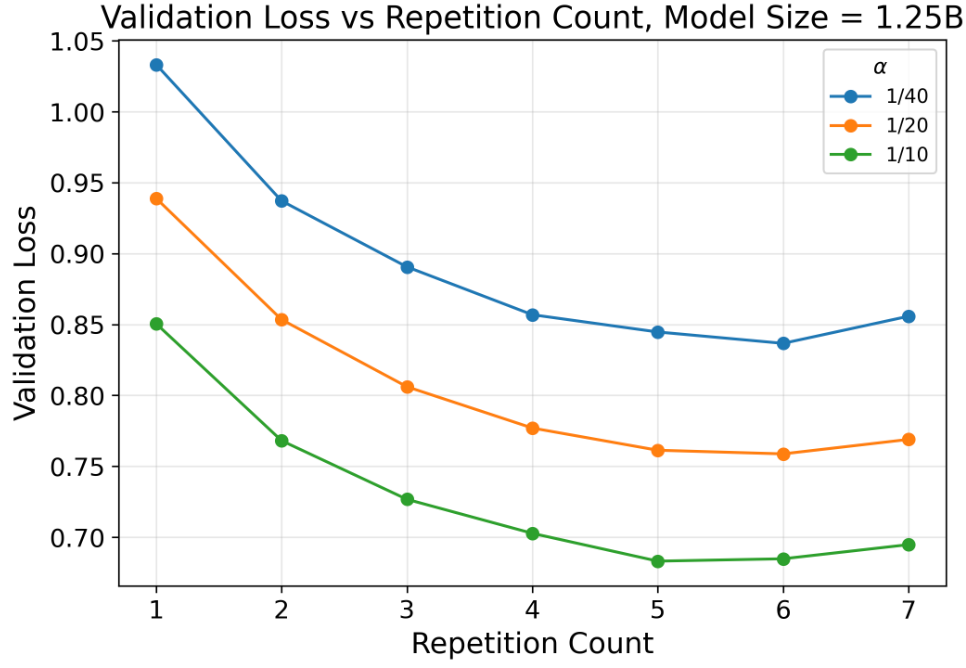} &
\includegraphics[width=0.23\textwidth]{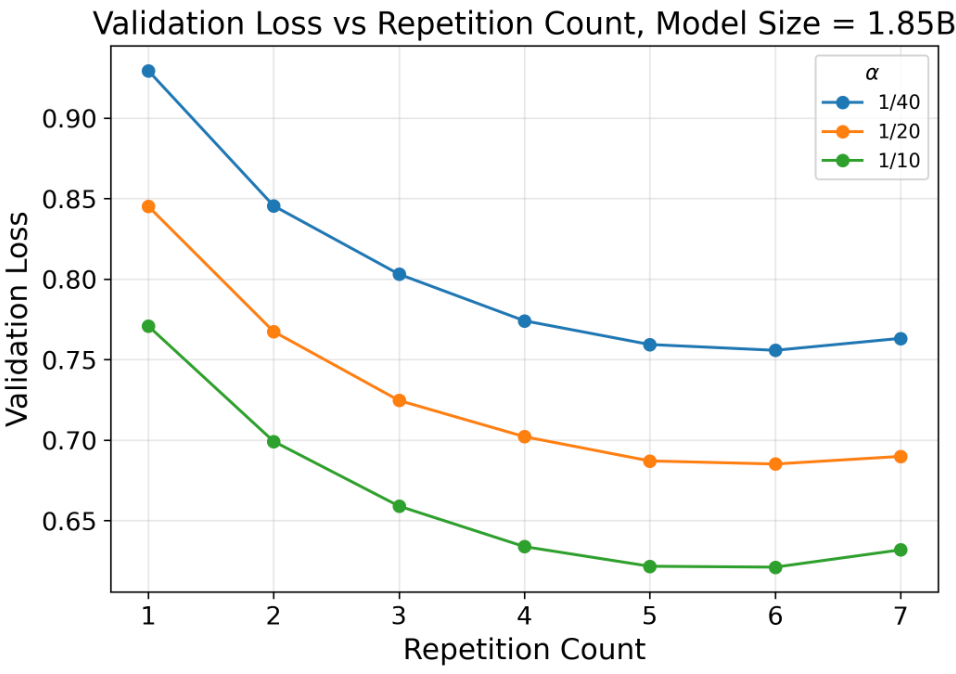} \\
\end{tabular}
\caption{
\textbf{Final validation loss across repetition counts for Math.}
Each panel corresponds to a different model size, and each curve represents a different fraction of unique high-quality data. All configurations are trained for the full token budget associated with the corresponding model size. The validation loss is generally minimized between 5 and 6 repetitions. Increasing the unique data fraction reduces the absolute validation loss but has little effect on the optimal repetition count, which increases only mildly with model size.
}
\label{fig:math_repetition}
\vspace{-0.3cm}
\end{figure*}

\begin{figure*}[h]\centering
\begin{tabular}{cccc}
\includegraphics[width=0.23\textwidth]{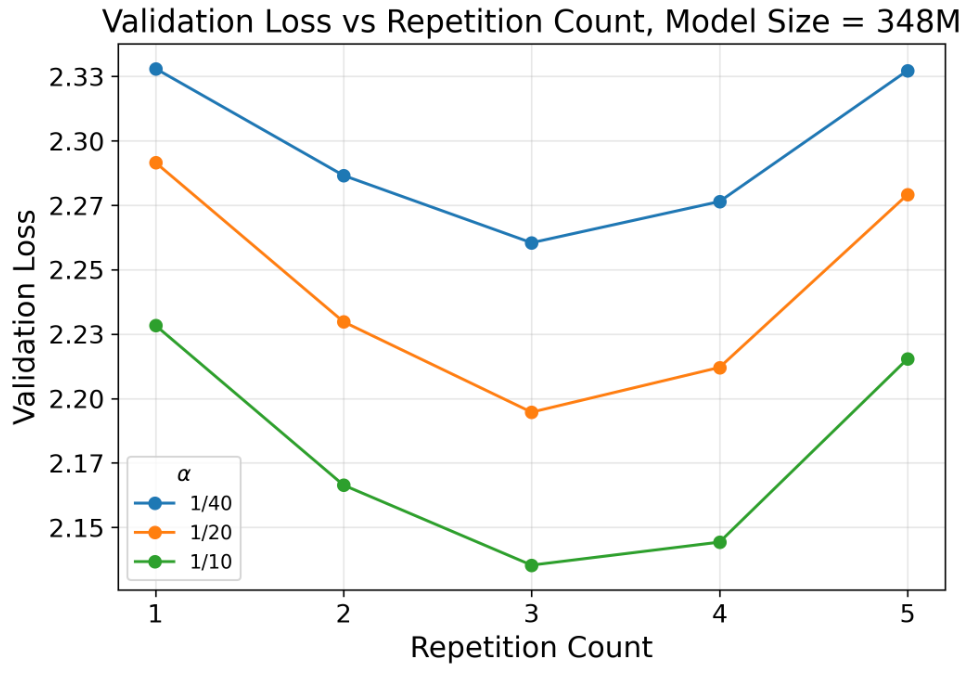} &
\includegraphics[width=0.23\textwidth]{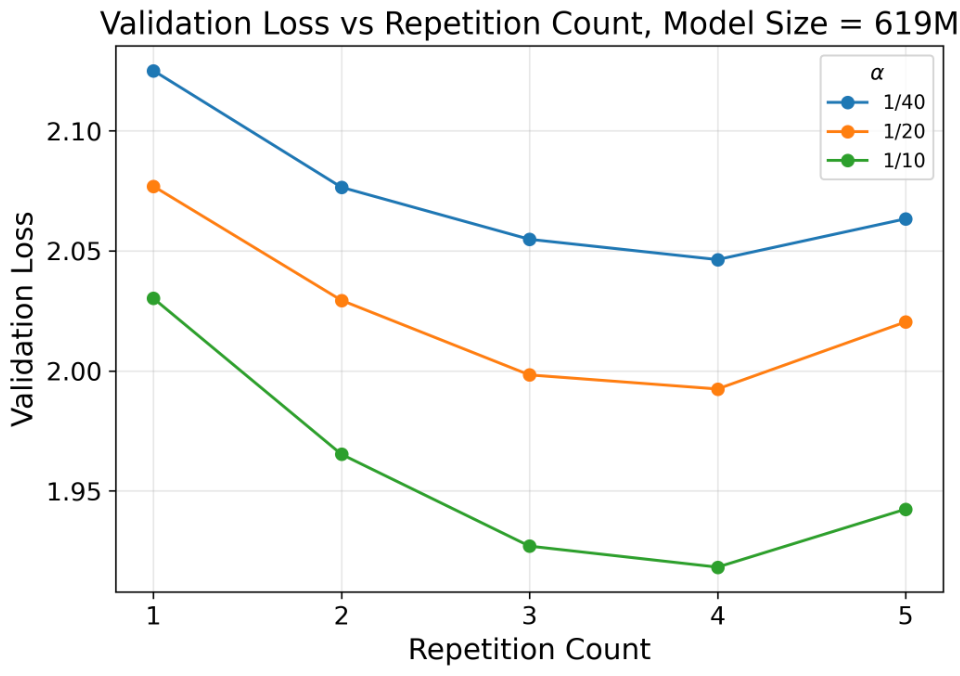} &
\includegraphics[width=0.23\textwidth]{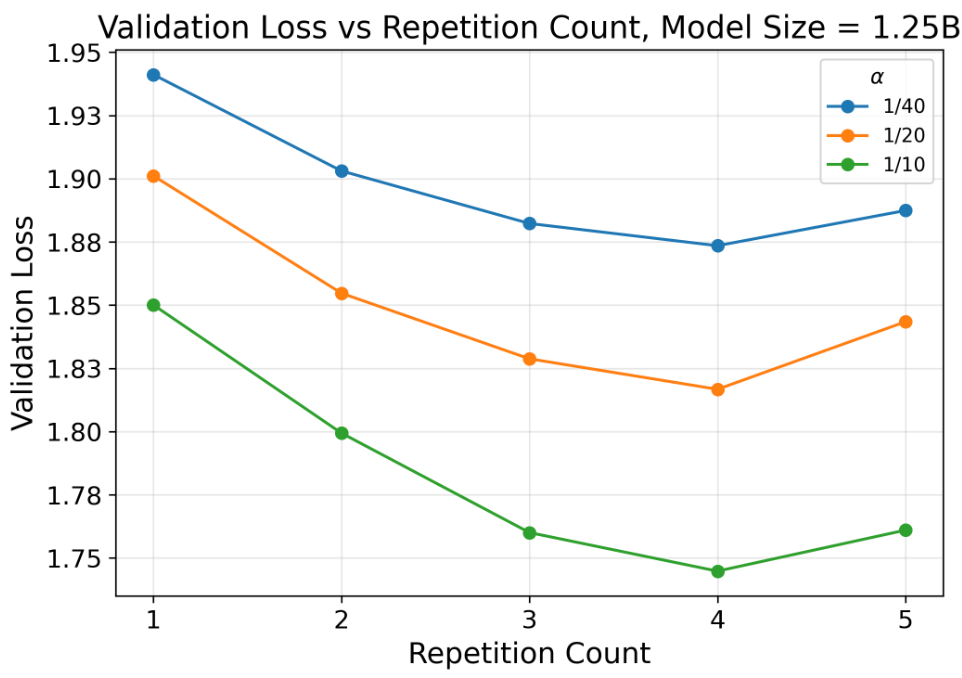} &
\includegraphics[width=0.23\textwidth]{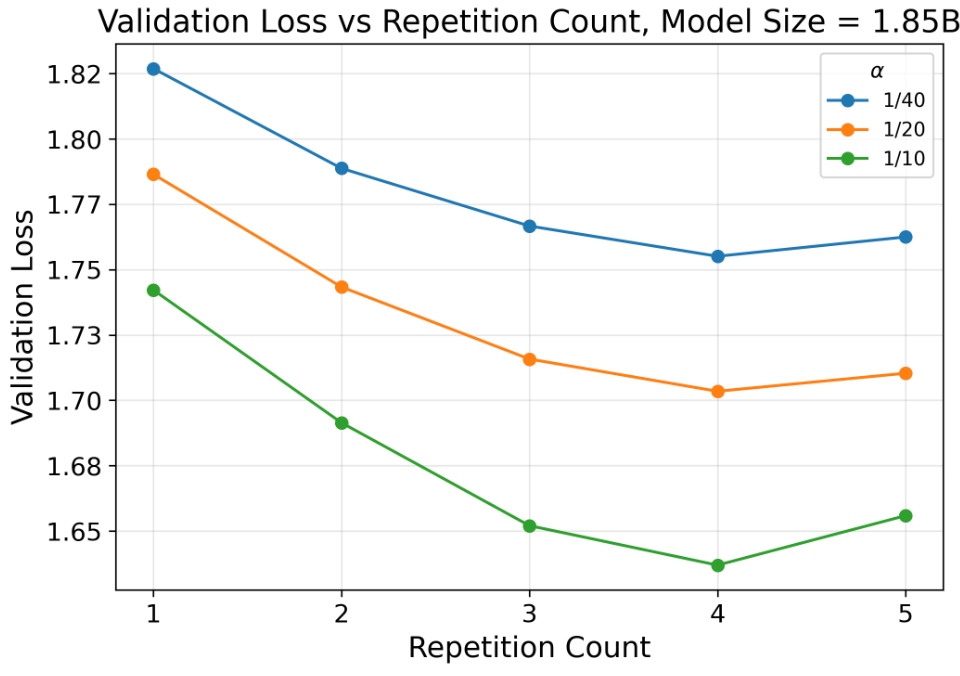} \\
\end{tabular}
\caption{
\textbf{Final validation loss across repetition counts for Wikipedia.}
Compared with Math, Wikipedia reaches its minimum validation loss earlier, typically after 3 to 4 repetitions, while showing the same stability across unique data fractions and mild increase with model size.
}
\label{fig:wiki_repetition}
\vspace{-0.3cm}
\end{figure*}

\begin{figure*}[h]
\centering
\begin{tabular}{cccc}
\includegraphics[width=0.23\textwidth]{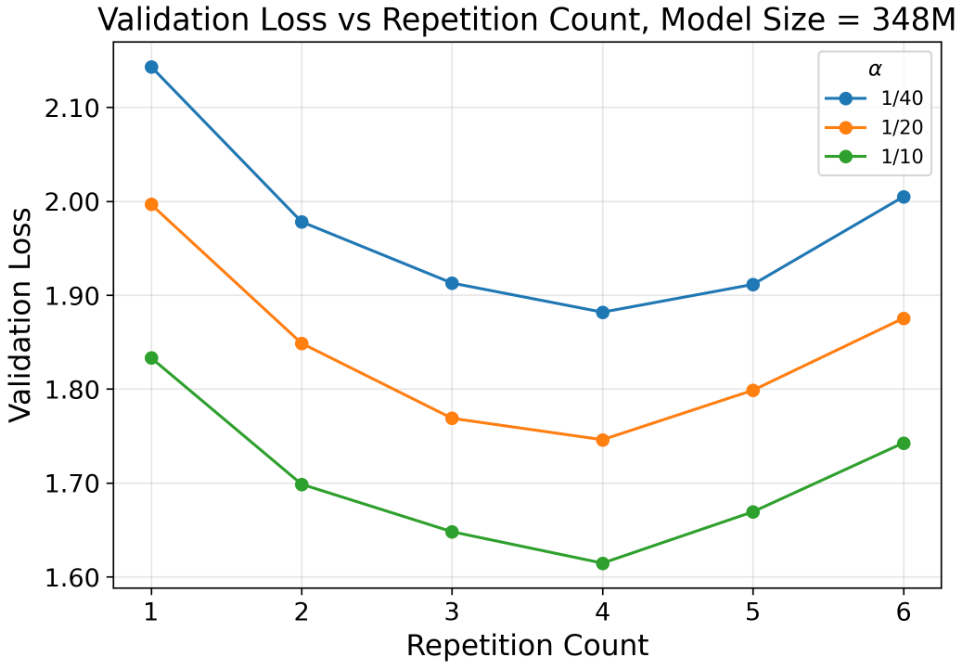} &
\includegraphics[width=0.23\textwidth]{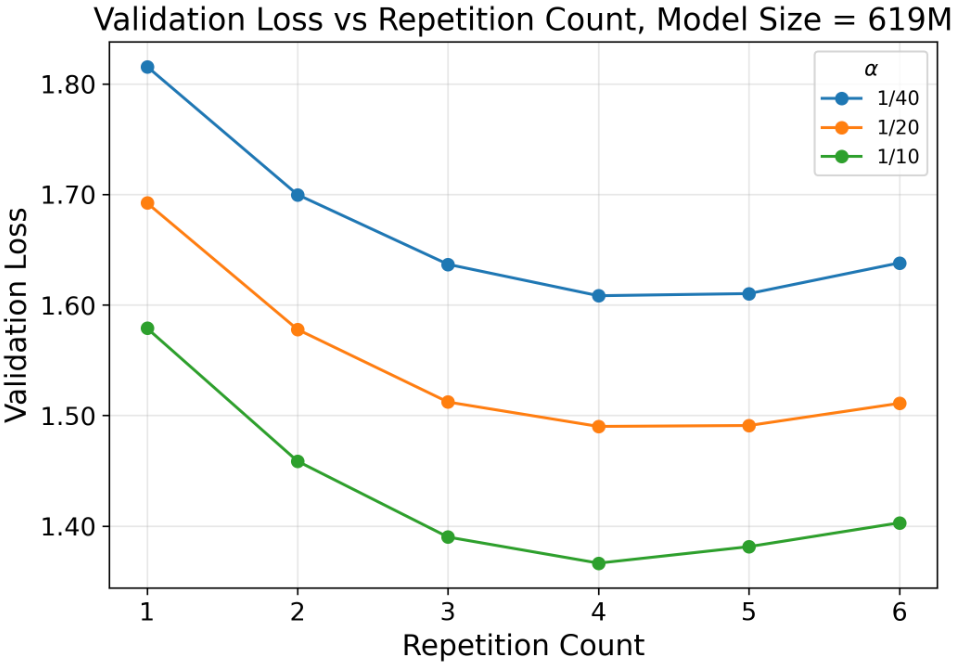} &
\includegraphics[width=0.23\textwidth]{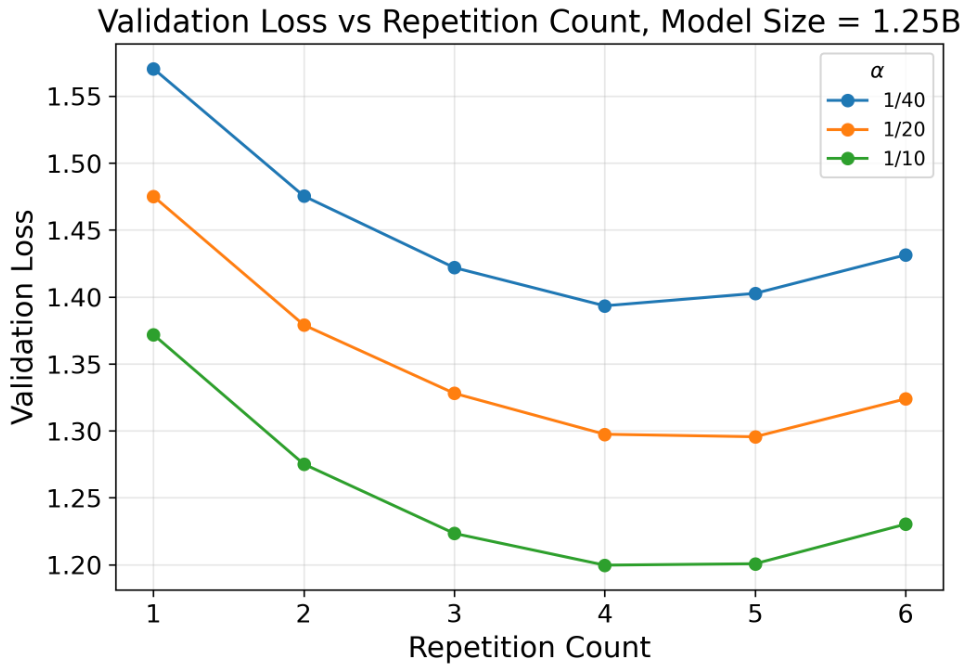} &
\includegraphics[width=0.23\textwidth]{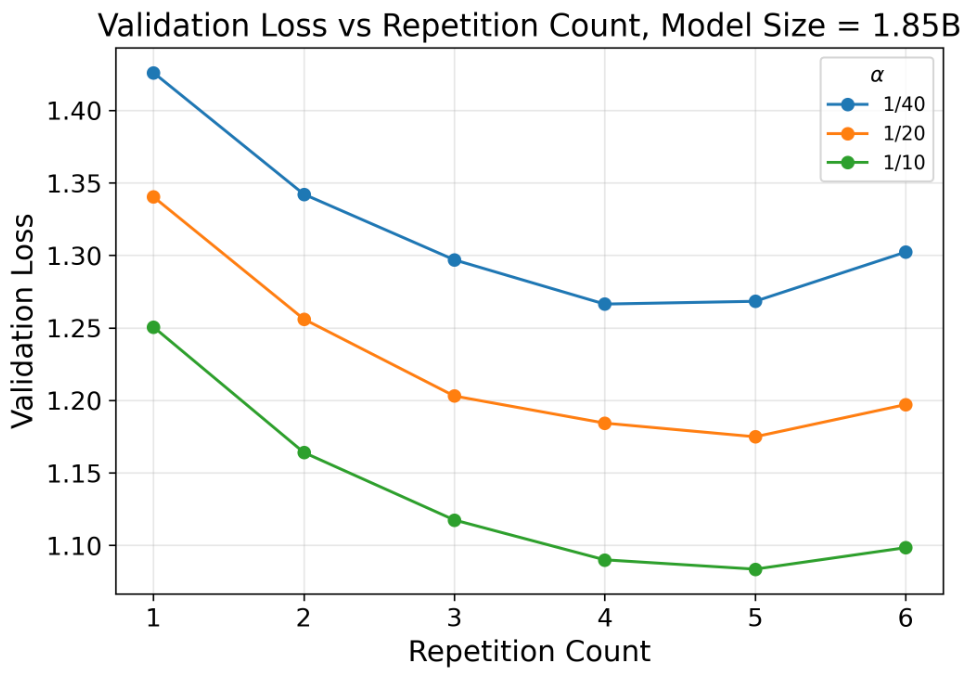} \\
\end{tabular}
\caption{
\textbf{Final validation loss across repetition counts for Code.}
Code generally reaches its minimum validation loss after 4 to 5 repetitions. The optimal repetition count remains stable across unique data fractions and increases mildly with model size.
}
\label{fig:code_repetition}
\vspace{-0.3cm}
\end{figure*}

\begin{figure*}[!t]
\centering
\begin{tabular}{cccc}
\includegraphics[width=0.23\textwidth]{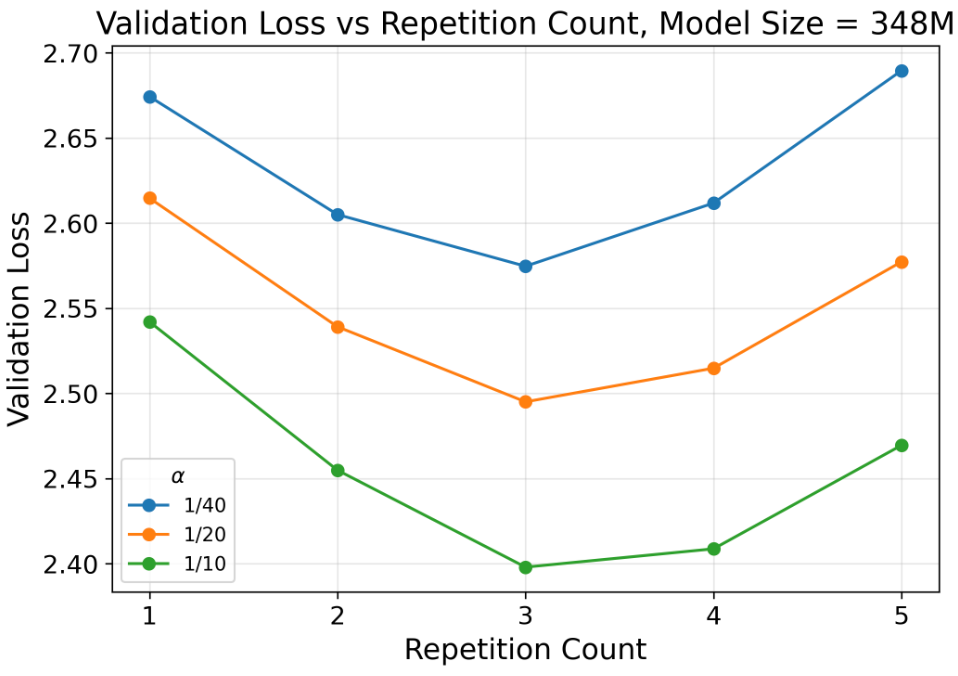} &
\includegraphics[width=0.23\textwidth]{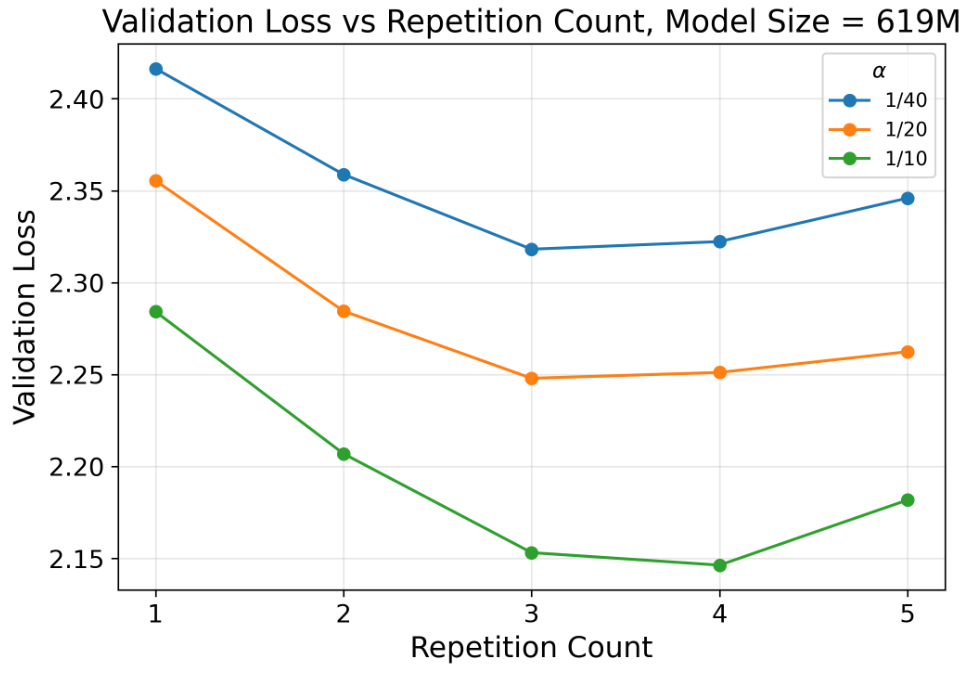} &
\includegraphics[width=0.23\textwidth]{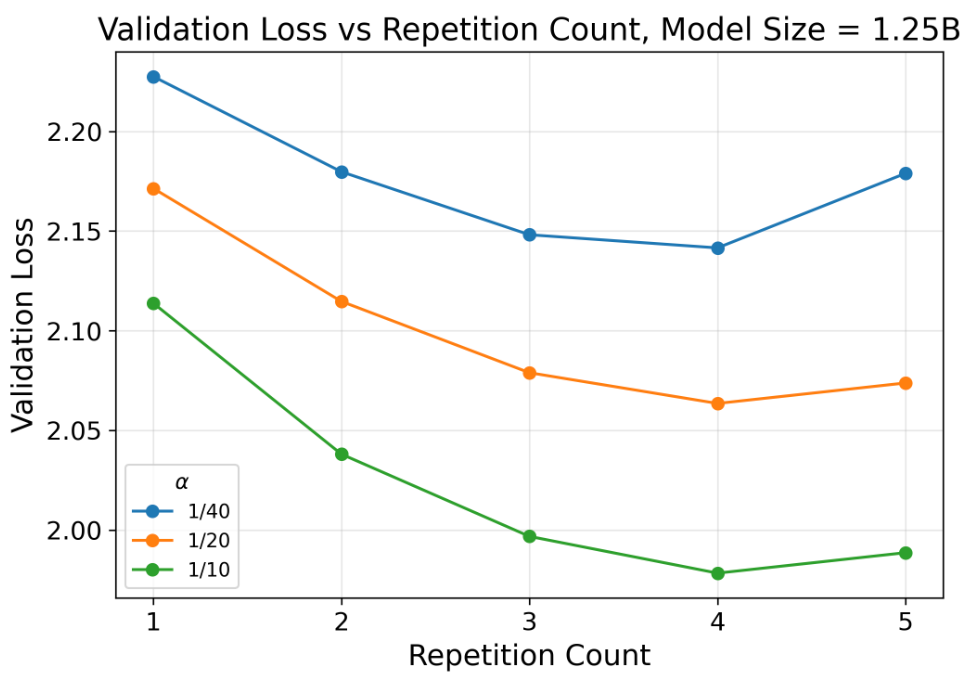} &
\includegraphics[width=0.23\textwidth]{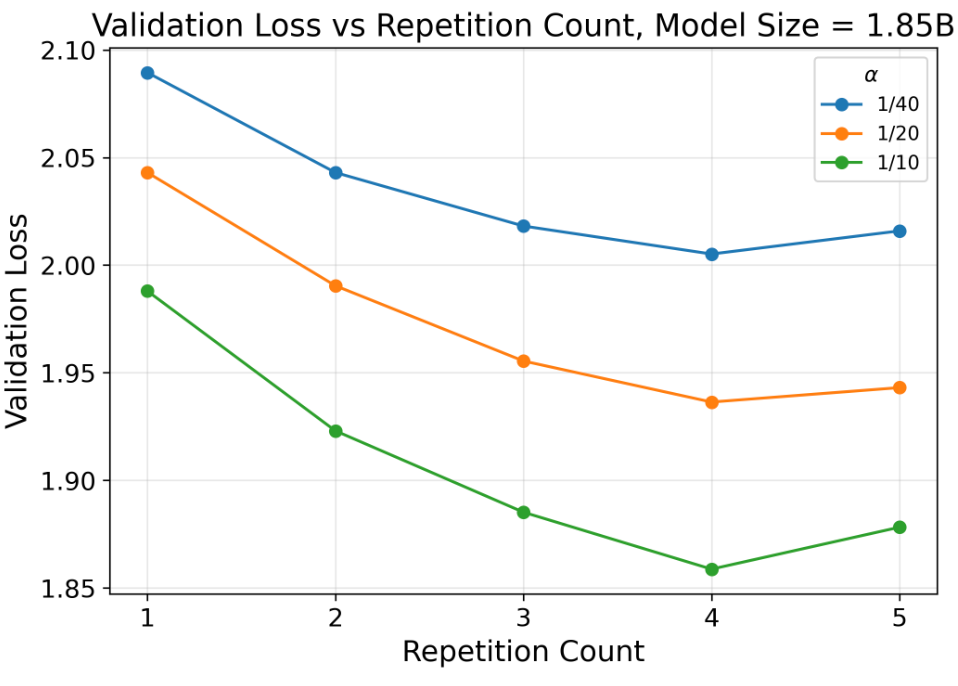} \\
\end{tabular}
\caption{
\textbf{Final validation loss across repetition counts for Medical.}
Medical generally reaches its minimum validation loss after 3 to 4 repetitions. As in the other domains, the optimal repetition count is largely insensitive to the unique data fraction and increases only mildly with model size.
}
\label{fig:medical_repetition}
\vspace{-0.3cm}
\end{figure*}

\FloatBarrier

%% file: appendix/experiments_5.tex
\section{Additional Results on Section~\ref{sec:fixed_concentration}}
\label{app:sec5}
In this section, we present additional results of Section~\ref{sec:fixed_concentration}.

Figure~\ref{fig:fixed_concentration_appendix} shows the results under the same fixed-token-budget setup as in the main text. For each curve, the total fraction of high-quality training tokens is fixed, while increasing the repetition count reduces the amount of unique high-quality data. This directly compares repeated data with an equivalent number of fresh unique tokens.

The results further show that the effectiveness of repetition differs across domains. Code and Medical behave similarly to Wikipedia, with validation loss increasing more noticeably as more unique data is replaced by repeated data. These results further support that whether repeated data can match fresh data under a fixed token budget depends strongly on the data domain.

\begin{figure*}[h]
\centering
\includegraphics[width=0.972\textwidth]{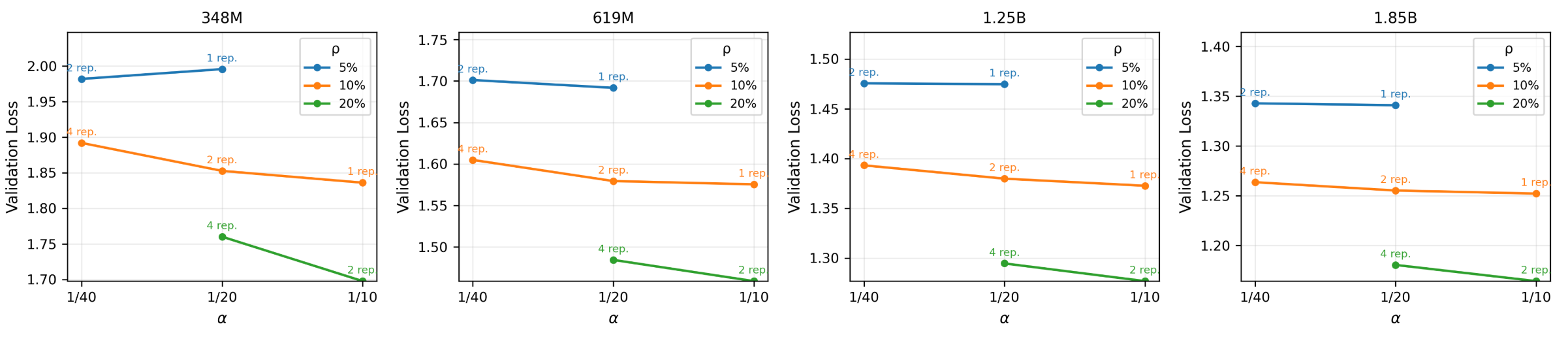} \\[4pt]
\includegraphics[width=0.972\textwidth]{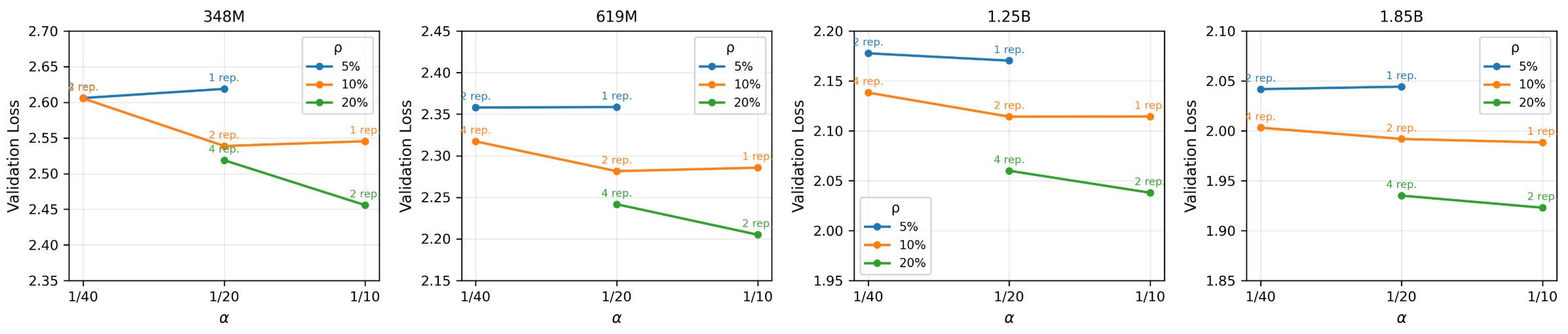}
\caption{
\textbf{Additional results at fixed total high-quality data fractions.}
Top: Code. Bottom: Medical. Code and Medical behave similarly to Wikipedia and shows a clearer degradation as more unique data is replaced by repeated data.
}
\label{fig:fixed_concentration_appendix}
\vspace{-0.3cm}
\end{figure*}

%% file: appendix/proof.tex
\newpage
\section{Proofs}
In this section, we provide the proofs for the results in Section~\ref{sec:theorem}.
\label{app:proof}

\subsection{Proof of Theorem \ref{thm:noise-decay}}

\begin{proof}[Proof of Theorem \ref{thm:noise-decay}]
For coordinate $k$, condition on the count $M_k=m$ of samples $\x_i=e_k$, where $M_k\sim\mathrm{Binomial}(D,p_k)$ with probability $\pi_{k,m}$. The exact averaged population risk is
\[
    2\pop_{D,N}(r;\beta,\sigma)
    =
    \sum_{k>N}k^{-\beta}
    +
    \sum_{k=1}^{N}\sum_{m=0}^{D}\pi_{k,m}
    \left[
        k^{-\beta}a_m^{2r}
        +
        p_k\frac{\sigma^2}{m}(1-a_m^r)^2
    \right],
\]
where $a_m=1-\eta m/D$, and the variance term is defined to be $0$ when $m=0$.

Let $F(r):=2\pop_{D,N}(r;\beta,\sigma_1)$ and $G(r):=2\pop_{D,N}(r;\beta,\sigma_2)$. Increasing the noise variance from $\sigma_1^2$ to $\sigma_2^2$ gives $G(r)=F(r)+\Delta(r)$, where
\[
    \Delta(r)
    =
    (\sigma_2^2-\sigma_1^2)
    \sum_{k=1}^{N}\sum_{m=1}^{D}
    \pi_{k,m}\frac{p_k}{m}(1-a_m^r)^2.
\]
The function $\Delta(r)$ is non-decreasing in $r$, and it is strictly increasing when $\sigma_2^2>\sigma_1^2$. Let $r_F$ and $r_G$ be the earliest minimizers of $F$ and $G$, respectively. Suppose that $r_G>r_F$. By the optimality of $r_F$ and the monotonicity of $\Delta$,
\[
    G(r_G)
    =F(r_G)+\Delta(r_G)
    \ge F(r_F)+\Delta(r_F)
    =G(r_F).
\]
If the inequality is strict, it contradicts the optimality of $r_G$. If equality holds, it contradicts $r_G$ being the earliest minimizer. Therefore, $r_G\le r_F$.

It remains to derive the small-noise rate. Write $\tau=\sigma^2$ and
\[
    H_\tau(r):=2\pop_{D,N}(r;\beta,\sigma).
\]
Its forward difference is
\[
\begin{aligned}
    H_\tau(r+1)-H_\tau(r)
    ={}&
    \sum_{k=1}^{N}\sum_{m=1}^{D}\pi_{k,m}
    (1-a_m)a_m^r \\
    &\quad\cdot
    \left[
        2p_k\frac{\tau}{m}
        -(1+a_m)
        \left(k^{-\beta}+p_k\frac{\tau}{m}\right)a_m^r
    \right].
\end{aligned}
\]
Let $a=a_1=1-\eta/D$. The contribution from $m=1$ can be written as $A\tau a^r-B_\tau a^{2r},$
where
\[
    A=2(1-a)\sum_{k=1}^{N}\pi_{k,1}p_k>0,
    \qquad
    B_\tau=(1-a)(1+a)
    \sum_{k=1}^{N}\pi_{k,1}(k^{-\beta}+p_k\tau)>0.
\]
If $D\ge2$, let $a_2=1-2\eta/D<a$; if $D=1$, the remainder below is identically zero. To evaluate $H_{\tau}(r+1)-H_{\tau}(r)$, we split it as follows:
$$H_{\tau}(r+1)-H_{\tau}(r)=A\tau a^r-B_{\tau}a^{2r}+\underbrace{\sum_{k=1}^{N}\sum_{m=2}^{D}\pi_{k,m}
    (1-a_m)a_m^r 
    \cdot
    \left[
        2p_k\frac{\tau}{m}
        -(1+a_m)
        \left(k^{-\beta}+p_k\frac{\tau}{m}\right)a_m^r
    \right]}_{:=E_{\tau}(r)}.$$

Since $D$ and $N$ are fixed, the contribution $E_\tau(r)$ from $m\ge2$ satisfies, uniformly for $0<\tau\le1$, $|E_\tau(r)|
    \le
    C\left(\tau a_2^r+a_2^{2r}\right)$
for a constant $C$ independent of $r$ and $\tau$.

By the optimality of $r_\tau$,
\[
    H_\tau(r_\tau+1)-H_\tau(r_\tau)\ge0,
    \qquad
    H_\tau(r_\tau)-H_\tau(r_\tau-1)\le0.
\]
Suppose first that $D\ge2$, and write $\rho=a_2/a<1$. Setting $x=a^{r_\tau}$ in the first inequality gives $\left(B_\tau-C\rho^{2r_\tau}\right)x
    \le
    \left(A+C\rho^{r_\tau}\right)\tau.$
Since $r_\tau\to\infty$ and $B_\tau\to B_0>0$, it follows that $a^{r_\tau}\le C_1\tau,$
for all sufficiently small $\tau$. Similarly, setting $y=a^{r_\tau-1}$ in the second inequality gives $\left(A-C\rho^{r_\tau-1}\right)\tau
    \le
    \left(B_\tau+C\rho^{2(r_\tau-1)}\right)y,$
and hence $a^{r_\tau}\ge C_2\tau$
for a constant $C_2>0$ independent of $\tau$. When $D=1$, the same two bounds follow directly because $E_\tau(r)=0$. Therefore, $a^{r_\tau}=\Theta(\tau).$
Taking logarithms and recalling that $\tau=\sigma^2$ yields
\[
    r^*(D,N;\beta,\sigma)
    =
    \frac{\log(1/\sigma^2)}{-\log(1-\eta/D)}+O(1).
\]
\end{proof}

\subsection{Proof of Theorem \ref{thm:model-size-stopping}}

\begin{proof}[Proof of Theorem \ref{thm:model-size-stopping}]
Let $\Delta_N(r):=2\pop_{D,N+1}(r)-2\pop_{D,N}(r)$ and set $j=N+1$. The exact risk increment is $\Delta_N(r)
    =
    \sum_{m=1}^{D}\pi_{j,m}h_{j,m}(r),$
where $h_{j,m}(r)
    =
    j^{-\beta}(a_m^{2r}-1)
    +
    c_\alpha j^{-\alpha}\frac{\sigma^2}{m}(1-a_m^r)^2.$
The forward difference is
\[
\begin{aligned}
    h_{j,m}(r+1)-h_{j,m}(r)
    ={}&(1-a_m)a_m^r \\
    &\quad\cdot
    \left[
        2c_\alpha j^{-\alpha}\frac{\sigma^2}{m}
        -(1+a_m)
        \left(j^{-\beta}+c_\alpha j^{-\alpha}\frac{\sigma^2}{m}\right)a_m^r
    \right].
\end{aligned}
\]
The bracketed term is minimized at $r=0$. Thus, $h_{j,m}(r+1)-h_{j,m}(r)\ge0$ for all $r\ge0$ provided that
\[
    c_\alpha j^{-\alpha}\frac{\sigma^2}{m}(1-a_m)
    \ge
    j^{-\beta}(1+a_m).
\]
Using $1-a_m=\eta m/D$ and $1+a_m\le2-\eta/D$ for $m\ge1$, a sufficient condition is
\[
    j^{\beta-\alpha}
    \ge
    \frac{D(2-\eta/D)}{c_\alpha\eta\sigma^2}.
\]
When $N>N_0(D,\sigma^2)$, this condition holds for $j=N+1$, implying that $\Delta_N(r)$ is non-decreasing in $r$.

Let $F(r):=2\pop_{D,N}(r)$ and $G(r):=2\pop_{D,N+1}(r)=F(r)+\Delta_N(r)$, with respective earliest minimizers $r_F$ and $r_G$. Suppose that $r_G>r_F$. By the optimality of $r_F$ and the monotonicity of $\Delta_N$,
\[
    G(r_G)
    =F(r_G)+\Delta_N(r_G)
    \ge F(r_F)+\Delta_N(r_F)
    =G(r_F).
\]
As in the previous proof, a strict inequality contradicts the optimality of $r_G$, while equality contradicts the earliest-minimizer convention. Therefore, $r_G\le r_F$.

The stated crossover order follows directly from
\[
    N_0(D,\sigma^2)
    =
    \left(
        \frac{2D-\eta}{c_\alpha\eta\sigma^2}
    \right)^{1/(\beta-\alpha)}.
\]
\end{proof}

\subsection{Proof of Theorem \ref{thm:linear-scaling}}

\begin{proof}[Proof of Theorem \ref{thm:linear-scaling}]
For coordinate $k$, let $M_k\sim\mathrm{Binomial}(D,p_k)$ and $a_m=1-\eta m/D$. From the exact risk formula, we have
\[
\begin{aligned}
2\pop_{D,N}(r+1)-2\pop_{D,N}(r)
={}&
-\underbrace{
\frac{\eta}{D}
\sum_{k=1}^{N}k^{-\beta}
\E\!\left[
    M_k(1+a_{M_k})a_{M_k}^{2r}
\right]
}_{S_D(r)}
\\
&+
\underbrace{
\frac{\eta\sigma^2}{D}
\sum_{k=1}^{N}p_k
\E\!\left[
    \left(
        2a_{M_k}^r-(1+a_{M_k})a_{M_k}^{2r}
    \right)
    \mathbf{1}_{\{M_k\ge1\}}
\right]
}_{V_D(r)}.
\end{aligned}
\]
Both $S_D(r)$ and $V_D(r)$ are nonnegative. The first term is the marginal signal benefit, whereas the second is the marginal variance penalty.

We first consider $1\ll r\le D$. By $\exp\left(-\frac{\eta rm}{D(1-\eta)}\right)
    \le a_m^r
    \le
    \exp\left(-\frac{\eta rm}{D}\right),$
 $\E[z^{M_k}]=(1-p_k+p_kz)^D \text{ and }
    \E[M_kz^{M_k}]
    =Dp_kz(1-p_k+p_kz)^{D-1}$
 the dominant coordinates satisfy $p_kr\asymp1$, or equivalently $k\asymp r^{1/\alpha}$. Standard sum--integral comparison then gives, uniformly in this range,
\[
    S_D(r)
    \asymp
    \sum_{k=1}^{N}k^{-(\alpha+\beta)}
    \exp(-c r k^{-\alpha})
    \asymp
    r^{\frac{1-\alpha-\beta}{\alpha}},
\]
and
\[
    V_D(r)
    \asymp
    \frac{\sigma^2}{D}
    \sum_{k=1}^{N}k^{-\alpha}
    \exp(-c r k^{-\alpha})
    \left(1-\exp(-c r k^{-\alpha})\right)
    \asymp
    \frac{\sigma^2}{D}
    r^{\frac{1-\alpha}{\alpha}}.
\]
Here and below, the positive constants denoted by $c$ may differ from line to line. Since $D/N\to C_0$ and $r\le D$, truncating the sums at $N$ does not change these orders. Therefore, $\frac{S_D(r)}{V_D(r)}
    \asymp
    \frac{D}{\sigma^2}r^{-\beta/\alpha}.$
The same estimates with one-sided constants also cover bounded $r$. Hence, there exist constants $0<c_0<C_0'<\infty$, independent of $D$, such that, for all sufficiently large $D$, $S_D(r)>V_D(r),
    0\le r\le c_0D^{\alpha/\beta},$  and $
    S_D(r)<V_D(r),
    C_0'D^{\alpha/\beta}\le r\le D.
$

It remains to exclude a later decrease when $r\ge D$. Let $a_1=1-\eta/D$. Restricting the variance sum to coordinates with $p_k\le D^{-1}$ and to the event $M_k=1$ gives $V_D(r)
    \ge
    c\sigma^2 a_1^rD^{\frac{1-2\alpha}{\alpha}}.$
Using $a_m\le a_1^m$ and the binomial identities above gives $S_D(r)
    \le
    C a_1^{2r}D^{\frac{1-\alpha-\beta}{\alpha}}.$
Consequently, $\frac{S_D(r)}{V_D(r)}
    \le
    C a_1^r\frac{D^{1-\beta/\alpha}}{\sigma^2}
    =o(1),$
where the last equality follows from $\beta>\alpha$. Thus, the exact risk is increasing for every $r\ge D$ when $D$ is sufficiently large.

We have shown that the exact risk is decreasing up to $c_0D^{\alpha/\beta}$ and increasing after $C_0'D^{\alpha/\beta}$. Therefore, its earliest global minimizer satisfies
\[
    c_0D^{\alpha/\beta}
    \le
    r^*(D,N)
    \le
    C_0'D^{\alpha/\beta},
\]
which proves
\[
    r^*(D,N)
    =
    \Theta\!\left(
        D^{\alpha/\beta}
    \right).
\]
\end{proof}

%% file: paper.bbl
\begin{thebibliography}{39}
\providecommand{\natexlab}[1]{#1}
\providecommand{\url}[1]{\texttt{#1}}
\expandafter\ifx\csname urlstyle\endcsname\relax
  \providecommand{\doi}[1]{doi: #1}\else
  \providecommand{\doi}{doi: \begingroup \urlstyle{rm}\Url}\fi

\bibitem[Bi et~al.(2024)Bi, Chen, Chen, Chen, Dai, Deng, Ding, Dong, Du, Fu, et~al.]{bi2024deepseek}
Xiao Bi, Deli Chen, Guanting Chen, Shanhuang Chen, Damai Dai, Chengqi Deng, Honghui Ding, Kai Dong, Qiushi Du, Zhe Fu, et~al.
\newblock Deepseek llm: Scaling open-source language models with longtermism.
\newblock \emph{arXiv preprint arXiv:2401.02954}, 2024.

\bibitem[Carlini et~al.(2022)Carlini, Ippolito, Jagielski, Lee, Tramer, and Zhang]{carlini2022quantifying}
Nicholas Carlini, Daphne Ippolito, Matthew Jagielski, Katherine Lee, Florian Tramer, and Chiyuan Zhang.
\newblock Quantifying memorization across neural language models.
\newblock In \emph{The Eleventh International Conference on Learning Representations}, 2022.

\bibitem[Charton and Kempe(2024)]{charton2024emergent}
Fran{\c{c}}ois Charton and Julia Kempe.
\newblock Emergent properties with repeated examples.
\newblock \emph{arXiv preprint arXiv:2410.07041}, 2024.

\bibitem[Chudnovsky et~al.(2026)Chudnovsky, Kazdan, Levi, Schaeffer, Denisov-Blanch, He, Donmez, Koyejo, and Donoho]{chudnovsky2026internal}
Jessica Chudnovsky, Joshua Kazdan, Noam Levi, Rylan Schaeffer, Yegor Denisov-Blanch, Bo~He, Mehmet Donmez, Sanmi Koyejo, and David Donoho.
\newblock Internal data repetition destroys language models.
\newblock \emph{arXiv preprint arXiv:2606.24998}, 2026.

\bibitem[Dai and Zheng(2026)]{dai2026explaining}
Rui Dai and Shuran Zheng.
\newblock Explaining data mixing scaling laws.
\newblock In \emph{Forty-third International Conference on Machine Learning}, 2026.
\newblock URL \url{https://openreview.net/forum?id=joReaAnwnH}.

\bibitem[Diao et~al.(2025)Diao, Yang, Fu, Dong, Su, Kliegl, Chen, Belcak, Suhara, Yin, et~al.]{diao2025climb}
Shizhe Diao, Yu~Yang, Yonggan Fu, Xin Dong, Dan Su, Markus Kliegl, Zijia Chen, Peter Belcak, Yoshi Suhara, Hongxu Yin, et~al.
\newblock Climb: Clustering-based iterative data mixture bootstrapping for language model pre-training.
\newblock \emph{arXiv preprint arXiv:2504.13161}, 2025.

\bibitem[Doddapaneni et~al.(2025)Doddapaneni, Ramesh, Khapra, Kunchukuttan, and Kumar]{doddapaneni2025primer}
Sumanth Doddapaneni, Gowtham Ramesh, Mitesh Khapra, Anoop Kunchukuttan, and Pratyush Kumar.
\newblock A primer on pretrained multilingual language models.
\newblock \emph{ACM Computing Surveys}, 57\penalty0 (9):\penalty0 1--39, 2025.

\bibitem[Dubey et~al.(2024)Dubey, Jauhri, Pandey, Kadian, Al-Dahle, Letman, Mathur, Schelten, Yang, Fan, et~al.]{dubey2024llama}
Abhimanyu Dubey, Abhinav Jauhri, Abhinav Pandey, Abhishek Kadian, Ahmad Al-Dahle, Aiesha Letman, Akhil Mathur, Alan Schelten, Amy Yang, Angela Fan, et~al.
\newblock The llama 3 herd of models.
\newblock \emph{arXiv e-prints}, pages arXiv--2407, 2024.

\bibitem[Fan et~al.(2024)Fan, Pagliardini, and Jaggi]{fan2024doge}
Simin Fan, Matteo Pagliardini, and Martin Jaggi.
\newblock Doge: Domain reweighting with generalization estimation.
\newblock In \emph{International Conference on Machine Learning}, pages 12895--12915. PMLR, 2024.

\bibitem[Gu et~al.(2025)Gu, Lyu, Li, and Zhang]{gu2025datamixinginducephase}
Xinran Gu, Kaifeng Lyu, Jiazheng Li, and Jingzhao Zhang.
\newblock Data mixing can induce phase transitions in knowledge acquisition.
\newblock \emph{arXiv preprint arXiv:2505.18091}, 2025.

\bibitem[Guo et~al.(2020)Guo, Dai, Vrande{\v{c}}i{\'c}, and Al-Rfou]{guo2020wiki}
Mandy Guo, Zihang Dai, Denny Vrande{\v{c}}i{\'c}, and Rami Al-Rfou.
\newblock Wiki-40b: Multilingual language model dataset.
\newblock In \emph{Proceedings of the Twelfth Language Resources and Evaluation Conference}, pages 2440--2452, 2020.

\bibitem[Hernandez et~al.(2022)Hernandez, Brown, Conerly, DasSarma, Drain, El-Showk, Elhage, Hatfield-Dodds, Henighan, Hume, et~al.]{hernandez2022scaling}
Danny Hernandez, Tom Brown, Tom Conerly, Nova DasSarma, Dawn Drain, Sheer El-Showk, Nelson Elhage, Zac Hatfield-Dodds, Tom Henighan, Tristan Hume, et~al.
\newblock Scaling laws and interpretability of learning from repeated data.
\newblock \emph{arXiv preprint arXiv:2205.10487}, 2022.

\bibitem[Hoffmann et~al.(2022)Hoffmann, Borgeaud, Mensch, Buchatskaya, Cai, Rutherford, de~Las~Casas, Hendricks, Welbl, Clark, et~al.]{hoffmann2022training}
Jordan Hoffmann, Sebastian Borgeaud, Arthur Mensch, Elena Buchatskaya, Trevor Cai, Eliza Rutherford, Diego de~Las~Casas, Lisa~Anne Hendricks, Johannes Welbl, Aidan Clark, et~al.
\newblock Training compute-optimal large language models.
\newblock In \emph{Proceedings of the 36th International Conference on Neural Information Processing Systems}, pages 30016--30030, 2022.

\bibitem[Kalra and Barkeshli(2026)]{kalra2026quantifying}
Dayal~Singh Kalra and Maissam Barkeshli.
\newblock Quantifying hyperparameter transfer and the importance of embedding layer learning rate.
\newblock \emph{arXiv preprint arXiv:2605.21486}, 2026.
\newblock \doi{10.48550/arXiv.2605.21486}.

\bibitem[Kang et~al.(2024)Kang, Sun, Wen, Chen, Song, Mahmood, and Jia]{kang2024autoscale}
Feiyang Kang, Yifan Sun, Bingbing Wen, Si~Chen, Dawn Song, Rafid Mahmood, and Ruoxi Jia.
\newblock Autoscale: Scale-aware data mixing for pre-training llms.
\newblock \emph{arXiv preprint arXiv:2407.20177}, 2024.

\bibitem[Kaplan et~al.(2020)Kaplan, McCandlish, Henighan, Brown, Chess, Child, Gray, Radford, Wu, and Amodei]{kaplan2020scaling}
Jared Kaplan, Sam McCandlish, Tom Henighan, Tom~B Brown, Benjamin Chess, Rewon Child, Scott Gray, Alec Radford, Jeffrey Wu, and Dario Amodei.
\newblock Scaling laws for neural language models.
\newblock \emph{arXiv preprint arXiv:2001.08361}, 2020.

\bibitem[Kazdan et~al.(2026)Kazdan, Levi, Schaeffer, Chudnovsky, Puri, He, Donmez, Koyejo, and Donoho]{kazdan2026scale}
Joshua Kazdan, Noam Levi, Rylan Schaeffer, Jessica Chudnovsky, Abhay Puri, Bo~He, Mehmet Donmez, Sanmi Koyejo, and David Donoho.
\newblock Scale dependent data duplication.
\newblock \emph{arXiv preprint arXiv:2603.06603}, 2026.

\bibitem[Lee et~al.(2022)Lee, Ippolito, Nystrom, Zhang, Eck, Callison-Burch, and Carlini]{lee2022deduplicating}
Katherine Lee, Daphne Ippolito, Andrew Nystrom, Chiyuan Zhang, Douglas Eck, Chris Callison-Burch, and Nicholas Carlini.
\newblock Deduplicating training data makes language models better.
\newblock In \emph{Proceedings of the 60th Annual Meeting of the Association for Computational Linguistics (Volume 1: Long Papers)}, pages 8424--8445, 2022.

\bibitem[Li et~al.(2025)Li, Chen, Huang, Wang, and Wu]{li2025functional}
Binghui Li, Fengling Chen, Zixun Huang, Lean Wang, and Lei Wu.
\newblock Functional scaling laws in kernel regression: Loss dynamics and learning rate schedules.
\newblock In \emph{The Thirty-ninth Annual Conference on Neural Information Processing Systems}, 2025.
\newblock URL \url{https://openreview.net/forum?id=dpllevHMbc}.

\bibitem[Li et~al.(2024)Li, Fang, Smyrnis, Ivgi, Jordan, Gadre, Bansal, Guha, Keh, Arora, et~al.]{li2024datacomp}
Jeffrey Li, Alex Fang, Georgios Smyrnis, Maor Ivgi, Matt Jordan, Samir Gadre, Hritik Bansal, Etash Guha, Sedrick Keh, Kushal Arora, et~al.
\newblock Datacomp-lm: In search of the next generation of training sets for language models.
\newblock \emph{Advances in Neural Information Processing Systems}, 37:\penalty0 14200--14282, 2024.

\bibitem[Li et~al.(2023)Li, Allal, Zi, Muennighoff, Kocetkov, Mou, Marone, Akiki, Li, Chim, et~al.]{li2023starcoder}
R~Li, LB~Allal, Y~Zi, N~Muennighoff, D~Kocetkov, C~Mou, M~Marone, C~Akiki, J~Li, J~Chim, et~al.
\newblock Starcoder: May the source be with you!
\newblock \emph{Transactions on machine learning research}, 2023.

\bibitem[Lin et~al.(2024)Lin, Wu, Kakade, Bartlett, and Lee]{lin2024scaling}
Licong Lin, Jingfeng Wu, Sham~M. Kakade, Peter Bartlett, and Jason~D. Lee.
\newblock Scaling laws in linear regression: Compute, parameters, and data.
\newblock In \emph{The Thirty-eighth Annual Conference on Neural Information Processing Systems}, 2024.
\newblock URL \url{https://openreview.net/forum?id=PH7sdEanXP}.

\bibitem[Liu et~al.(2026)Liu, Zhou, Liu, Guo, Wang, Zhang, Zhang, Yu, Zhou, and Wang]{liu2026infolaw}
Fengze Liu, Weidong Zhou, Binbin Liu, Ping Guo, Zijun Wang, Bingni Zhang, Yifan Zhang, Yifeng Yu, Xiaohuan Zhou, and Taifeng Wang.
\newblock Infolaw: Information scaling laws for large language models with quality-weighted mixture data and repetition.
\newblock \emph{arXiv preprint arXiv:2605.02364}, 2026.

\bibitem[Liu et~al.(2025)Liu, Su, Yao, Jiang, Lai, Du, Qin, Xu, Lu, Yan, et~al.]{liu2025muon}
Jingyuan Liu, Jianlin Su, Xingcheng Yao, Zhejun Jiang, Guokun Lai, Yulun Du, Yidao Qin, Weixin Xu, Enzhe Lu, Junjie Yan, et~al.
\newblock Muon is scalable for llm training.
\newblock \emph{arXiv preprint arXiv:2502.16982}, 2025.

\bibitem[Liu et~al.(2024)Liu, Zheng, Muennighoff, Zeng, Dou, Pang, Jiang, and Lin]{liu2024regmix}
Qian Liu, Xiaosen Zheng, Niklas Muennighoff, Guangtao Zeng, Longxu Dou, Tianyu Pang, Jing Jiang, and Min Lin.
\newblock Regmix: Data mixture as regression for language model pre-training.
\newblock In \emph{The Thirteenth International Conference on Learning Representations}, 2024.

\bibitem[Lovelace et~al.(2026)Lovelace, Belardi, Kundurthy, Sudhakar, and Weinberger]{lovelace2026prescriptive}
Justin Lovelace, Christian Belardi, Srivatsa Kundurthy, Shriya Sudhakar, and Kilian~Q Weinberger.
\newblock Prescriptive scaling laws for data constrained training.
\newblock \emph{arXiv preprint arXiv:2605.01640}, 2026.

\bibitem[Muennighoff et~al.(2023)Muennighoff, Rush, Barak, Le~Scao, Tazi, Piktus, Pyysalo, Wolf, and Raffel]{muennighoff2023scaling}
Niklas Muennighoff, Alexander Rush, Boaz Barak, Teven Le~Scao, Nouamane Tazi, Aleksandra Piktus, Sampo Pyysalo, Thomas Wolf, and Colin~A Raffel.
\newblock Scaling data-constrained language models.
\newblock \emph{Advances in Neural Information Processing Systems}, 36:\penalty0 50358--50376, 2023.

\bibitem[Paster et~al.(2024)Paster, Dos~Santos, Azerbayev, and Ba]{paster2024openwebmath}
Keiran Paster, Marco Dos~Santos, Zhangir Azerbayev, and Jimmy Ba.
\newblock Openwebmath: An open dataset of high-quality mathematical web text.
\newblock In \emph{International Conference on Learning Representations}, volume 2024, pages 20357--20379, 2024.

\bibitem[Penedo et~al.(2024)Penedo, Kydl{\'\i}{\v{c}}ek, Lozhkov, Mitchell, Raffel, Von~Werra, Wolf, et~al.]{penedo2024fineweb}
Guilherme Penedo, Hynek Kydl{\'\i}{\v{c}}ek, Anton Lozhkov, Margaret Mitchell, Colin Raffel, Leandro Von~Werra, Thomas Wolf, et~al.
\newblock The fineweb datasets: Decanting the web for the finest text data at scale.
\newblock \emph{Advances in Neural Information Processing Systems}, 37:\penalty0 30811--30849, 2024.

\bibitem[Raffel et~al.(2020)Raffel, Shazeer, Roberts, Lee, Narang, Matena, Zhou, Li, and Liu]{raffel2020exploring}
Colin Raffel, Noam Shazeer, Adam Roberts, Katherine Lee, Sharan Narang, Michael Matena, Yanqi Zhou, Wei Li, and Peter~J Liu.
\newblock Exploring the limits of transfer learning with a unified text-to-text transformer.
\newblock \emph{Journal of machine learning research}, 21\penalty0 (140):\penalty0 1--67, 2020.

\bibitem[Sedova et~al.(2026)Sedova, Seto, Schluter, and Ablin]{sedova2026scaling}
Anastasiia Sedova, Skyler Seto, Natalie Schluter, and Pierre Ablin.
\newblock Scaling laws for mixture pretraining under data constraints.
\newblock \emph{arXiv preprint arXiv:2605.12715}, 2026.

\bibitem[Shukor et~al.(2025)Shukor, Bethune, Busbridge, Grangier, Fini, El-Nouby, and Ablin]{shukor2025scalinglawsoptimaldata}
Mustafa Shukor, Louis Bethune, Dan Busbridge, David Grangier, Enrico Fini, Alaaeldin El-Nouby, and Pierre Ablin.
\newblock Scaling laws for optimal data mixtures.
\newblock \emph{arXiv preprint arXiv:2507.09404}, 2025.

\bibitem[Taylor et~al.(2022)Taylor, Kardas, Cucurull, Scialom, Hartshorn, Saravia, Poulton, Kerkez, and Stojnic]{taylor2022galactica}
Ross Taylor, Marcin Kardas, Guillem Cucurull, Thomas Scialom, Anthony Hartshorn, Elvis Saravia, Andrew Poulton, Viktor Kerkez, and Robert Stojnic.
\newblock Galactica: A large language model for science.
\newblock \emph{arXiv preprint arXiv:2211.09085}, 2022.

\bibitem[Team et~al.(2023)Team, Anil, Borgeaud, Alayrac, Yu, Soricut, Schalkwyk, Dai, Hauth, Millican, et~al.]{team2023gemini}
Gemini Team, Rohan Anil, Sebastian Borgeaud, Jean-Baptiste Alayrac, Jiahui Yu, Radu Soricut, Johan Schalkwyk, Andrew~M Dai, Anja Hauth, Katie Millican, et~al.
\newblock Gemini: a family of highly capable multimodal models.
\newblock \emph{arXiv preprint arXiv:2312.11805}, 2023.

\bibitem[Wu et~al.(2026)Wu, Xiao, Okanovic, Sternal, van Keulen, Pechenizkiy, Mocanu, Hoefler, and Mocanu]{wu2026data}
Boqian Wu, Qiao Xiao, Patrik Okanovic, Tomasz Sternal, Maurice van Keulen, Mykola Pechenizkiy, Elena Mocanu, Torsten Hoefler, and Decebal~Constantin Mocanu.
\newblock When data is scarce: Scaling sparse language models with repeated training.
\newblock \emph{arXiv preprint arXiv:2606.01155}, 2026.

\bibitem[Xie et~al.(2023)Xie, Pham, Dong, Du, Liu, Lu, Liang, Le, Ma, and Yu]{xie2023doremi}
Sang~Michael Xie, Hieu Pham, Xuanyi Dong, Nan Du, Hanxiao Liu, Yifeng Lu, Percy~S Liang, Quoc~V Le, Tengyu Ma, and Adams~Wei Yu.
\newblock Doremi: Optimizing data mixtures speeds up language model pretraining.
\newblock \emph{Advances in Neural Information Processing Systems}, 36:\penalty0 69798--69818, 2023.

\bibitem[Xue et~al.(2023)Xue, Fu, Zhou, Zheng, and You]{xue2023repeat}
Fuzhao Xue, Yao Fu, Wangchunshu Zhou, Zangwei Zheng, and Yang You.
\newblock To repeat or not to repeat: Insights from scaling llm under token-crisis.
\newblock \emph{Advances in Neural Information Processing Systems}, 36:\penalty0 59304--59322, 2023.

\bibitem[Yan et~al.(2025)Yan, Wen, Li, Luo, Chen, and Lyu]{yan2025larger}
Tingkai Yan, Haodong Wen, Binghui Li, Kairong Luo, Wenguang Chen, and Kaifeng Lyu.
\newblock Larger datasets can be repeated more: A theoretical analysis of multi-epoch scaling in linear regression.
\newblock \emph{arXiv preprint arXiv:2511.13421}, 2025.

\bibitem[Ye et~al.(2025)Ye, Liu, Sun, Zhan, Zhou, and Qiu]{ye2025data}
Jiasheng Ye, Peiju Liu, Tianxiang Sun, Jun Zhan, Yunhua Zhou, and Xipeng Qiu.
\newblock Data mixing laws: Optimizing data mixtures by predicting language modeling performance.
\newblock In \emph{The Thirteenth International Conference on Learning Representations}, 2025.

\end{thebibliography}
